\documentclass[11pt]{article}

\usepackage[final]{acl}

\usepackage{times}
\usepackage{latexsym}

\usepackage[T1]{fontenc}

\usepackage[utf8]{inputenc}

\usepackage{microtype}

\usepackage{inconsolata}

\usepackage{graphicx}

\newcommand{\cmark}{\textcolor{green}{\ding{51}}} 
\newcommand{\xmark}{\textcolor{red}{\ding{55}}}   

\usepackage[utf8]{inputenc}
\usepackage[T1]{fontenc}
\usepackage{algorithm}
\usepackage{algorithmic}
\usepackage{graphicx}
\usepackage{booktabs}
\usepackage{multirow}
\usepackage{amsmath}
\usepackage{amsfonts}
\usepackage{enumitem}
\usepackage{url} 
\usepackage{xspace}
\usepackage{xcolor}
\usepackage{hyperref}
\usepackage{epigraph}
\usepackage{amsthm}
\usepackage{amssymb}
\usepackage{arydshln}
\usepackage{amsmath}
\usepackage{bbm}
\usepackage{subcaption}
\usepackage{wrapfig}
\usepackage{pifont}

\theoremstyle{plain}
\newtheorem{theorem}{Theorem}
\newtheorem{lemma}{Lemma}

\theoremstyle{definition}
\newtheorem{definition}{Definition}
\newtheorem{assumption}{Assumption}

\theoremstyle{remark}

\usepackage[dvipsnames]{xcolor}
\usepackage{soul}
\usepackage[most]{tcolorbox}

\usepackage[most]{tcolorbox}

\colorlet{lightSalmon}{Salmon!80}

\newcommand{\colorize}[2]{%
  \ifnum#1>50
    \edef\colorval{\the\numexpr (#1-50)*3/5+10\relax}%
    \colorbox{red!\colorval}{\strut #2}%
  \else
    \edef\colorval{\the\numexpr (50-#1)*3/5+10\relax}%
    \colorbox{blue!\colorval}{\strut #2}%
  \fi
}

\definecolor{langdark}{rgb}{0, 0, 0}
\definecolor{langlightblue}{rgb}{0.3, 0.65, 1}
\definecolor{langblue}{rgb}{0, 0.4, 0.8}
\definecolor{langwildblue}{rgb}{0.0, 0.45, 0.73}
\definecolor{langdarkblue}{rgb}{0.0, 0.0, 0.61}
\definecolor{langred}{rgb}{0.81, 0.09, 0.13}
\definecolor{langgreen}{rgb}{0.0, 0.6, 0.3}
\definecolor{binglightpink}{rgb}{1.0, 0.6, 0.71}
\definecolor{bingpink}{rgb}{1.0, 0.35, 0.71}
\definecolor{bingdarkpink}{rgb}{1.0, 0.25, 0.71}
\definecolor{bingdarkdarkpink}{rgb}{1.0, 0, 0.71}
\hypersetup{
    colorlinks=true,
    linkcolor=langdark,
    citecolor=binglightpink,
    urlcolor=bingpink,
}

\makeatletter
\AtBeginDocument{%
  \let\oldref\ref
  \renewcommand{\ref}[1]{%
    \hyperref[{#1}]{\underline{\oldref{#1}}}%
  }%
}
\makeatother

\usepackage{titletoc}

\newcommand\DoToC{%
  \twocolumn[
    \startcontents
    \printcontents{}{1}{%
      \textbf{\large Contents of Appendix}\vskip3pt\hrule\vskip5pt
    }
    \vskip3pt\hrule\vskip5pt
  ]
}

\title{Not All Tokens Matter: Towards Efficient LLM Reasoning via Token Significance in Reinforcement Learning}

\author{
Hanbing Liu$^1$\thanks{Work done during internship at Microsoft.}\thanks{Affiliated with the Shenzhen Key Laboratory of Ubiquitous Data Enabling, Tsinghua Shenzhen International Graduate School, Tsinghua University.} \quad
Lang Cao$^2$\footnotemark[1] \quad
Yuanyi Ren$^3$\footnotemark[1] \quad
Mengyu Zhou$^4$\thanks{Corresponding author. zhoumengyu.zmy@alibaba-inc.com.} \quad Haoyu Dong$^5$  \\
\textbf{
Xiaojun Ma$^5$ \quad
Shi Han$^5$ \quad
Dongmei Zhang$^5$} \\
$^1$Tsinghua University \quad
$^2$University of Illinois Urbana-Champaign \quad 
$^3$Peking University \\
$^4$Qwen Large Model Application Team, Alibaba \quad
$^5$Microsoft \\
\texttt{liuhb24@mails.tsinghua.edu.cn, langcao2@illinois.edu, yyren@pku.edu.cn} \\
}

\begin{document}
\maketitle
\begin{abstract}
Large language models (LLMs) show strong reasoning abilities but often produce unnecessarily long explanations that reduce efficiency. Although reinforcement learning (RL) has been used to improve reasoning, most methods focus on accuracy and rely on uniform length-based rewards that overlook the differing contributions of individual tokens, often harming correctness. We revisit length optimization in RL through the perspective of \textit{token significance}. Observing that many chain-of-thought (CoT) tokens contribute little to the final answer, we introduce a significance-aware length reward that selectively penalizes insignificance tokens, reducing redundancy while preserving essential reasoning. We also propose a dynamic length reward that encourages more detailed reasoning early in training and gradually shifts toward conciseness as learning progresses. Integrating these components into standard policy optimization yields a framework that improves both reasoning efficiency and accuracy. Experiments across multiple benchmarks demonstrate substantial reductions in response length while preserving or improving correctness, highlighting the importance of modeling token significance for efficient LLM reasoning. Code is available at \url{https://github.com/microsoft/Bingo}.
\end{abstract}


\begin{figure}[t]
    \centering
    \includegraphics[width=\linewidth]{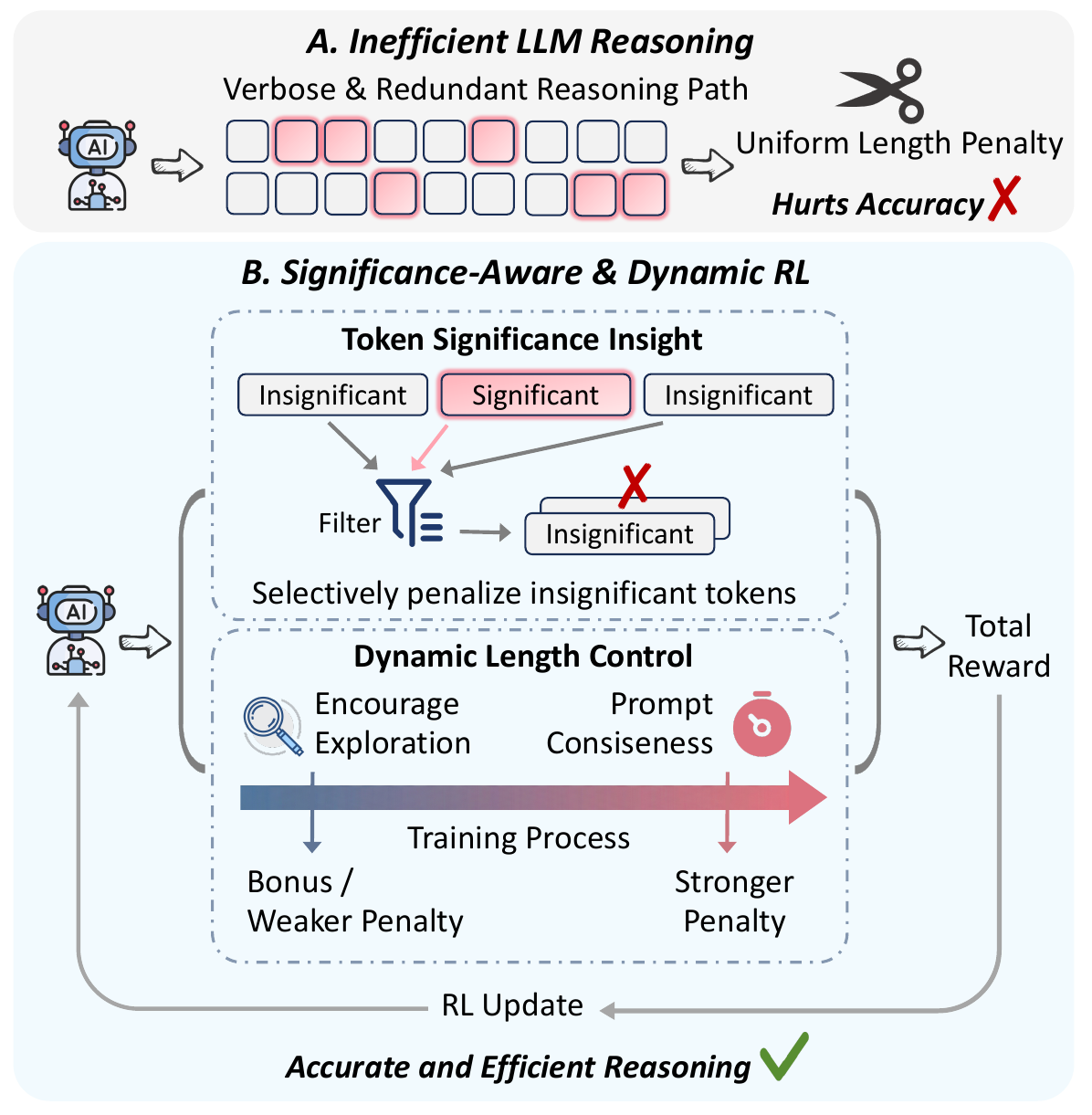}
    \caption{
    \textbf{(A) Limitations of Uniform Length Penalization.}
    LLMs often produce verbose and redundant reasoning, and applying a uniform length penalty, as done in many prior RL approaches, fails to account for the differing importance of individual tokens, which can lead to accuracy degradation;
    \textbf{(B) Reinforcement Learning with Token Significance and Dynamic Length Control.}
    Our method models token significance to selectively penalize unimportant tokens and introduces dynamic length control to balance exploration and conciseness throughout training, enabling LLMs to generate reasoning that is both accurate and efficient.
    }
    \label{fig:teaser}
    \vspace{-1.1em}

\end{figure}

\section{Introduction}

Large language models (LLMs) \citep{openai2024gpt4technicalreport,gunasekar2023textbooksneed} have demonstrated impressive reasoning capabilities across a variety of tasks, from arithmetic problem solving \citep{uesato2022solvingmathwordproblems,hendrycksmath2021,aime_1983_2024} to commonsense reasoning \citep{chen-etal-2023-theoremqa}. A key observation from recent work is that sufficiently large models can exhibit emergent reasoning abilities, such as chain-of-thought (CoT) reasoning \citep{wei2022chain}, without explicit supervision~\citep{wei2022emergentabilitieslargelanguage, suzgun2022challengingbigbenchtaskschainofthought}. Despite these successes, a major challenge persists: LLMs often generate unnecessarily verbose or redundant reasoning traces, leading to inefficiencies in computational cost, redundancy, and latency.

Improving reasoning efficiency of LLMs has thus emerged as an important research direction~\citep{qu2025surveyefficientreasoninglarge, sui2025stopoverthinkingsurveyefficient, li202512surveyreasoning, wang2025harnessingreasoningeconomysurvey}. Prior work in this area can be broadly categorized into supervised fine-tuning (SFT) approaches \citep{xia2025tokenskipcontrollablechainofthoughtcompression, xu2025twtthinkingtokenshabitual, zhang2025lightthinkerthinkingstepbystepcompression, kang2024c3otgeneratingshorterchainofthought} and reinforcement learning (RL) approaches \citep{luo2025o1prunerlengthharmonizingfinetuningo1like, kimiteam2025kimik15scalingreinforcement, arora2025traininglanguagemodels, aggarwal2025l1controllinglongreasoning, yeo2025demystifyinglongchainofthoughtreasoning, shi2025dastdifficultyadaptiveslowthinking}. SFT-based methods focus on constructing compressed reasoning traces and training models to imitate them. While these approaches can be effective, they rely on high-quality compressed supervision, which is costly to obtain and often lacks generalizability across diverse tasks. RL-based methods typically introduce length-based rewards that penalize overly long responses to encourage brevity. However, the design of such rewards or penalties in RL-based methods remains underexplored and is often overly simplistic. For example, O1-Pruner~\citep{luo2025o1prunerlengthharmonizingfinetuningo1like} applies a uniform penalty to all samples, assuming that every response should be shortened. This assumption often leads to performance degradation, as not all reasoning traces are equally verbose—some require more detailed steps to arrive at the correct answer. To address this, other works have proposed more selective penalty strategies, conditioning penalties on sample correctness~\citep{qu2025surveyefficientreasoninglarge,kimiteam2025kimik15scalingreinforcement,yeo2025demystifyinglongchainofthoughtreasoning} or estimated difficulty~\citep{shi2025dastdifficultyadaptiveslowthinking}. These approaches typically assign stronger penalties to simpler questions and weaker ones to more challenging cases. However, accurately estimating question difficulty remains a fundamental challenge, and unresolved hard questions often lead to unnecessarily long responses, further undermining reasoning efficiency.

Despite growing interest, current designs of length-based rewards remain limited, as they often fail to adequately promote concise reasoning while preserving answer accuracy. For example, prior work has largely overlooked the impact of token-level contributions on the overall efficiency of reasoning as illustrated in Figure~\ref{fig:teaser} (A). In this work, we approach the problem from a novel perspective grounded in the concept of \textit{\textbf{token significance}}. Our motivation arises from observed token redundancy in LLMs~\citep{hou-etal-2022-token,lin2025rho1tokensneed}, where many tokens in chain-of-thought (CoT) reasoning contribute little to the final answer. We posit that \textit{not all tokens are equally important for efficient reasoning}—many are insignificant, such as redundant phrases or unnecessary intermediate steps, and can be removed without degrading performance. Existing reward designs often overlook this distinction. In contrast, we introduce a \textbf{\textit{significance-aware length reward}} that selectively penalizes only those insignificant tokens which do not meaningfully contribute to the final answer, while preserving essential reasoning steps.

We also observe that effectively handling hard questions is essential for efficient reasoning. Prior work~\citep{muennighoff2025s1simpletesttimescaling,wu2025lessunderstandingchainofthoughtlength} has shown that encouraging extended CoT reasoning can improve performance by enabling deeper exploration, which may help solve more difficult questions. Therefore, it is intuitive to use length as an incentive for hard questions. However, LLMs should solve difficult questions not only accurately but also concisely. Applying a static length incentive can lead to unnecessarily long responses, which may still fail to produce correct answers. To address this, we incorporate a \textbf{\textit{dynamic length reward}} that adapts over the course of training. This reward is applied to significant tokens in incorrect samples to balance exploration and efficiency. Specifically, it encourages longer reasoning in the early training phase to promote exploration, and gradually shifts toward penalizing excessive length in later stages to promote conciseness.

Building on these insights, we introduce \textbf{\textsc{Bingo}} (\textbf{B}oosting Efficient Reason\textbf{ING} in Policy \textbf{O}ptimization), an RL framework that incorporates our two proposed reward mechanisms into standard RL algorithms such as Proximal Policy Optimization (PPO)~\citep{Schulman2017ProximalPO} as illustrated in Figure~\ref{fig:teaser} (B). This enables joint optimization of both reasoning accuracy and efficiency. Extensive experiments across diverse reasoning benchmarks show that \textsc{Bingo} consistently outperforms strong baselines by reducing redundant computation while maintaining or improving accuracy (Appendix~\ref{ap:baseline}).

In summary, this paper makes the following key contributions:
\begin{itemize}[leftmargin=*, itemsep=0pt, labelsep=5pt, topsep=0pt]
    \item \textbf{Token Significance Insight.}
        We introduce the concept of \textit{token significance} in policy optimization, distinguishing between \textit{significant} and \textit{insignificant} tokens in reasoning traces. This insight motivates our \textit{significance-aware length reward}, which explicitly penalizes uninformative tokens while preserving critical reasoning content, enabling more targeted and effective length control.
    \item \textbf{Dynamic Length Control.}
        We propose a \textit{dynamic length reward} strategy that adjusts the reward signal over the course of training—encouraging longer reasoning in the early stages to foster exploration, and gradually promoting conciseness as the model converges.
    \item \textbf{Efficiency-Oriented RL Framework.}
        We develop \textbf{\textsc{Bingo}}, a new reinforcement learning framework that integrates both reward strategies. Extensive experiments across multiple reasoning benchmarks, along with comprehensive analyses, demonstrate its effectiveness.
\end{itemize}

\section{Related Work}

\noindent \textbf{Reinforcement Learning for Large Language Models.}
Reinforcement Learning (RL) \citep{Kaelbling1996ReinforcementLA} has emerged as a powerful paradigm for aligning large language models (LLMs) with human preferences. In Reinforcement Learning from Human Feedback (RLHF) \citep{Christiano2017DeepRL, Stiennon2020LearningTS, Ouyang2022TrainingLM}, the Proximal Policy Optimization (PPO) algorithm \citep{Schulman2017ProximalPO} is employed alongside human preference data to train a reward model that steers the fine-tuning of LLMs. Building on PPO, subsequent works like GRPO \citep{shao2024deepseekmathpushinglimitsmathematical} and REINFORCE++ \citep{Hu2025REINFORCEAS} have proposed improved variants to address its limitations. 
Beyond alignment, RL has also shown promise in improving the reasoning capabilities of LLMs. Early studies \citep{lightman2023letsverifystepstep, uesato2022solvingmathwordproblems} demonstrated that reward-guided training can enhance multi-step reasoning performance. More recently, DeepSeek-R1 \citep{deepseekai2025deepseekr1incentivizingreasoningcapability} demonstrated that large-scale RL can substantially boost reasoning ability across a wide range of tasks, pointing to a promising direction for future work. \\

\noindent \textbf{Efficient Reasoning with Large Language Models.}
Recent advances have empowered language models to perform strong reasoning via inference-time techniques such as chain‑of‑thought prompting \citep{wei2023chainofthoughtpromptingelicitsreasoning, yao2023treethoughtsdeliberateproblem, cao-2024-graphreason, wang2023selfconsistencyimproveschainthought} and post-training \citep{lightman2023letsverifystepstep, uesato2022solvingmathwordproblems, deepseekai2025deepseekr1incentivizingreasoningcapability,jiang2025deepretrievalhackingrealsearch, cao2025fortune}. More recent work has shifted to optimizing both accuracy and efficiency. Some approaches improve efficiency at inference time, such as token‑budget‑aware reasoning \citep{han2025tokenbudgetawarellmreasoning}, or prompting strategies like “reason‑without‑thinking” \citep{ma2025reasoningmodelseffectivethinking} and chain‑of‑draft \citep{xu2025chaindraftthinkingfaster}. Others apply post‑training optimization via supervised fine‑tuning (SFT), including TokenSkip \citep{xia2025tokenskipcontrollablechainofthoughtcompression}, TwT \citep{xu2025twtthinkingtokenshabitual}, LightThinker \citep{zhang2025lightthinkerthinkingstepbystepcompression}, and C3oT \citep{kang2024c3otgeneratingshorterchainofthought}. These SFT methods primarily construct high‑quality compressed reasoning paths containing key information, and train the models on them.
In parallel, RL-based approaches often improve efficiency by incorporating length controls or penalties into their reward functions \citep{luo2025o1prunerlengthharmonizingfinetuningo1like, kimiteam2025kimik15scalingreinforcement, arora2025traininglanguagemodels, aggarwal2025l1controllinglongreasoning, yeo2025demystifyinglongchainofthoughtreasoning, shi2025dastdifficultyadaptiveslowthinking}. For instance, O1-Pruner \citep{luo2025o1prunerlengthharmonizingfinetuningo1like} uses offline length rewards comparing samples against mean lengths. Kimi k1.5 \citep{kimiteam2025kimik15scalingreinforcement} applies online penalties to correct samples only. 
Building on prior RL-based approaches, we advance length-based reward design to enable LLMs to balance reasoning accuracy with computational efficiency.

\begin{figure*}[ht!]
    \centering
    \includegraphics[width=\textwidth]{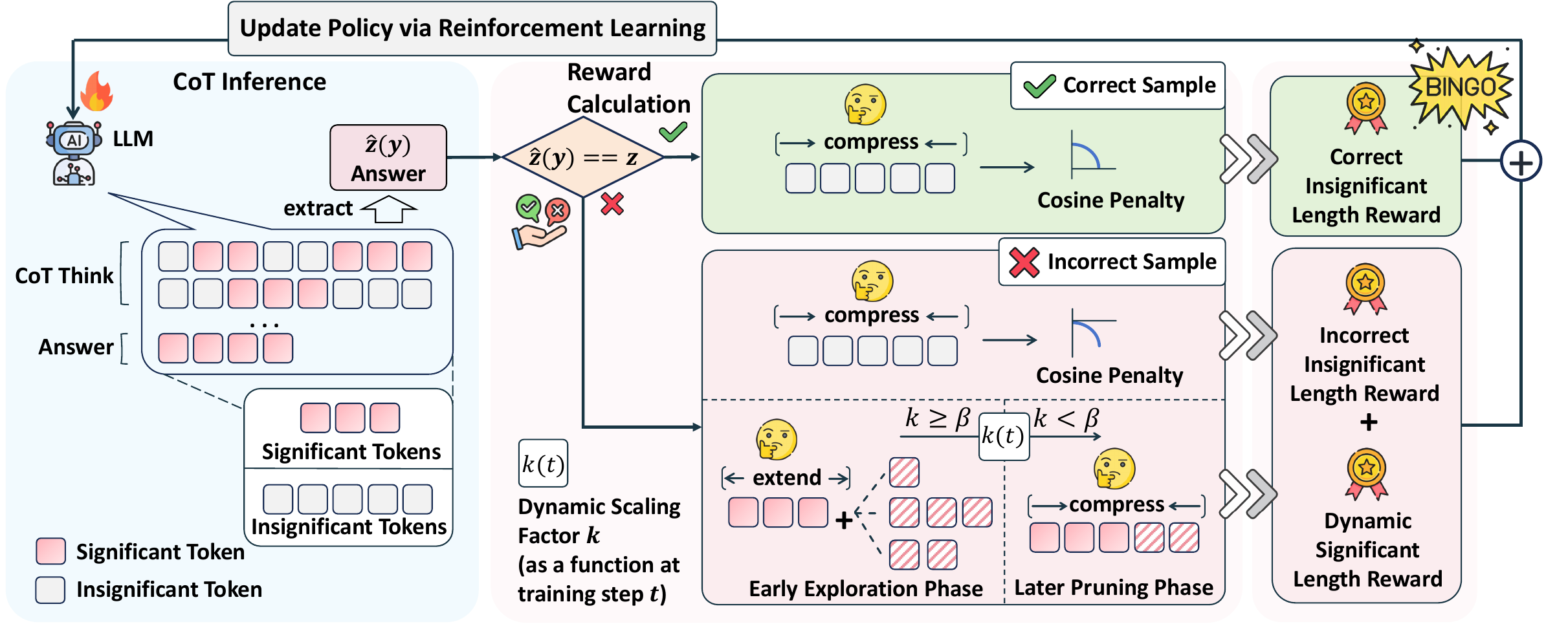}
    \caption{\textbf{Illustration of the \textsc{Bingo} framework.} Given a generated CoT trace, the LLM first distinguishes between \textit{significant} and \textit{insignificant} tokens. A dynamic length reward is then computed based on token type and sample correctness. During the early exploration phase of training (\(k(t) \geq \beta\)), the reward encourages extended reasoning for significant tokens in incorrect samples while penalizing insignificant tokens in all cases. As training progresses (\(k(t) < \beta\)), the reward shifts toward promoting conciseness by discouraging both significant and insignificant length where appropriate. This two-stage strategy allows the model to first explore broadly and then compress effectively. The aggregated rewards are then used to update the policy via RL, resulting in more accurate and efficient reasoning.}
    \label{fig:bingo}
    \vspace{-1em}
\end{figure*}

\section{Methodology}
\label{sec:method}
In this section, we introduce the design of the significance-aware length reward and the dynamic length reward, and explain how these two reward mechanisms are integrated into the \textsc{Bingo} framework, as illustrated in Figure~\ref{fig:bingo}. All notations are list at Appendix~\ref{ap:notation}.

\subsection{Task Formulation}
\label{sec:task}
\noindent \textbf{Chain of Thought Reasoning.} Let $x$ denote a prompt, and let $y = (y_1, y_2, \dots, y_n)$ represent the sequence generated by a language model parameterized by $\theta$, where $y_i$ is the $i$-th token in the sequence, and $n$ is the total length of the sequence. Tokens are generated autoregressively from the conditional distribution:
\begin{equation}
\pi_\theta(y \mid x) = \prod_{t=1}^{n} \pi_\theta  \bigl(y_i \mid x,\,y_{1:i-1}\bigr),
\end{equation}
where the product runs over all tokens in the sequence, with each token $y_i$ (i.e, the action $a_i$) conditioned on the prompt $x$ and the previous tokens $y_{1:i-1}$ (i.e., the state $s_i$). Generation continues until an end-of-sequence (EOS) token is produced, signaling the completion of the response. During this process, the model may produce intermediate reasoning tokens, referred to as a \emph{chain of thought} (CoT) \citep{wei2022chain}, before generating the final answer. Therefore, the full output sequence, denoted as $y$, consists of both the chain of thought and the final answer.



\noindent \textbf{Optimization Objective of Efficient Reasoning.}
The performance of the model in efficient reasoning is assessed along two key dimensions: \textit{accuracy} and \textit{efficiency}.

Accuracy is measured by the \emph{Exact Match} (EM) metric, which evaluates whether the model's final answer matches the ground truth. Let \( \hat{z}(y) \) denote the final answer extracted from the model-generated sequence \( y \), typically corresponding to its final segment. Let \( z \) be the ground-truth answer. Then EM is defined as:
\begin{equation}
\text{EM} = \mathbb{E}_{x \sim \phi} \mathbb{E}_{y \sim \pi_\theta(\cdot|x)} \mathbbm{1}\left[ \hat{z}(y) = z \right],
\end{equation}
where \( \phi \) denotes the distribution over prompts. The indicator function returns 1 if the predicted answer exactly matches the ground truth, and 0 otherwise.

Efficiency is measured by the \emph{response length} $L$, typically defined as the number of tokens $n$ in the generated sequence $y = (y_1, y_2, \dots, y_n)$. While longer sequences may offer detailed reasoning, they often result in higher computational cost. Thus, reducing unnecessary tokens without harming accuracy is crucial for practical deployment. An ideal model achieves high EM while minimizing the average response length $L$, striking a balance between correctness and conciseness.

\noindent \textbf{Token Significance.}
CoT responses consist of a sequence of tokens that contribute unequally to the overall reasoning process. This observation motivates a token-level view of reasoning efficiency, in which the contribution of each token is described by the notion of \emph{token significance}.

Given a generated sequence 
\( y = (y_1, y_2, \dots, y_n) \),
the significance of a token \( y_i \) is defined via a contextual scoring function:
\begin{equation}
\label{eq:significance_task_formulation}
S(y_i) = F\!\left(y_i \mid \mathbf{y}_{\le n}\right),
\end{equation}
where \( F(\cdot) \) assigns a scalar score to a token based on its surrounding context within the full sequence, and \( \mathbf{y}_{\le n} \) denotes the complete set of tokens in the generated response. Tokens associated with higher scores are considered to have higher significance. Details of the scoring function and a discussion of alternative definitions of token significance are presented in Appendix~\ref{ap:sign_measure}.

\subsection{Significance-Aware Length Reward}
\label{sec:significance}



To enhance the efficiency of CoT generation, it is crucial to recognize that not all tokens in a CoT sequence contribute equally to deriving the final answer. Tokens with low importance scores are considered insignificant, while those with high scores are deemed significant. Specifically, we classify tokens as follows:
\begin{equation}
\text{Token } y_i \text{ is } 
\begin{cases}
\text{insignificant}, & \text{if } S(y_i) < \tau, \\
\text{significant}, & \text{if } S(y_i) \geq \tau.
\end{cases}
\end{equation}
\noindent
We then compute the total number of significant tokens $L^s$ and insignificant tokens $L^{is}$ in the response as:
\begin{equation}
L^s = \sum_{i=1}^{n} \mathbbm{1}\left( S(y_i) \geq \tau \right),
L^{is} = \sum_{i=1}^{n} \mathbbm{1}\left( S(y_i) < \tau \right),
\end{equation}
\noindent
where \( \mathbbm{1}(\cdot) \) is the indicator function, and $\tau$ is a pre-defined threshold. 
To encourage brevity while maintaining reasoning quality, we introduce a \textit{\textbf{significance-aware length reward}} that penalizes the excessive use of insignificant tokens through a cosine-based decay:
\begin{equation}
r_{is}(y) = \cos\left( \operatorname{clip}\left( \frac{L^{is}}{L_{\text{ref}}^{is}},\, 0,\, \frac{\pi}{2} \right) \right) + \mathbbm{1}[\hat{z}(y) = z]
\end{equation}
\noindent
where \( L_{\text{ref}}^{is} \) denotes the number of insignificant tokens in a reference response. The cosine function ensures a smooth, non-linear penalty that gradually decreases the reward as \( L^{is} \) increases, while the clipping operation bounds the angle to the interval \([0, \frac{\pi}{2}]\), preventing negative rewards. The final reward combines this length-based penalty with an answer reward derived from the EM indicator, ensuring that answer correctness is preserved.

This reward formulation ensures that shorter or equally concise responses—measured in terms of insignificant content—receive higher rewards, while excessively verbose outputs are gently penalized. 
Notably, our approach preserves natural fluency and coherence in generated text by constraining only the aggregate length of insignificant tokens, without dictating specific token selections or sequences in RL-based training.
Compared to standard length-based penalties, our significance-aware approach achieves equal or greater length reductions with less accuracy degradation by selectively penalizing insignificant tokens, as theoretically justified in Appendix~\ref{ap:theory}.

\subsection{Dynamic Length Reward for Significant Tokens}
\label{sec:dynamic}

While insignificant tokens are consistently penalized to reduce redundancy, significant tokens warrant a more nuanced approach. In the early stages of training, allowing longer reasoning with significant content can facilitate exploration and support the development of robust problem-solving strategies. However, as training progresses, conciseness becomes increasingly important for improving efficiency.

To accommodate this shift, we introduce a \textit{\textbf{dynamic length reward}} for significant tokens that evolves over time based on the model's learning trajectory. This adaptive mechanism is guided by a dynamic scaling factor that captures trends in accuracy and modulates the reward accordingly. Formally, the length-based reward for significant tokens is defined as:
\begin{equation}
r_s(y) = 
\begin{cases} 
k \cdot  \frac{L^s}{L_{\text{ref}}^s}, & \text{if } k \geq \beta \\
-\alpha \cdot t \cdot  \frac{L^s}{L_{\text{ref}}^s}, & \text{if } k < \beta
\end{cases}
\end{equation}
\noindent
where \( L^s \) represents the number of significant tokens in the generated output, \( L_{\text{ref}}^s \) is the corresponding value from the reference model, and \( k \) is a dynamic scaling factor that reflects the reasoning trend during training. The training step \( t \) begins at 1 and increments gradually when \( k \) first falls below the threshold \( \beta \), which determines when the model transitions from incentivizing longer significant token lengths to penalizing them. $\alpha$ is a weight that determines the rate of decay in this process. The value of \( k \) is estimated by fitting a linear model to recent training steps:

\begin{equation}
k = \frac{\sum_{t = S_a}^{S_b} (t - \bar{t})(acc_t - \overline{acc})}{\sum_{t = S_a}^{S_b} (t - \bar{t})^2}
\end{equation}
\noindent
where \( acc_t \) denotes the training batch accuracy at training step \( t \), \( \bar{t} \) and \( \overline{acc} \) are the mean step index and mean accuracy over the interval \( [S_a, S_b] \). A positive \( k \) indicates an upward accuracy trend, suggesting that the model is still in an improvement phase. As training progresses and accuracy plateaus, \(  k \) approaches zero or becomes negative. The theoretical rationale behind the design of our dynamic length reward schedule is discussed in detail in Appendix \ref{ap:discussion_dynamic}. This dynamic adaptation allows the model to balance early-stage exploration with late-stage compression, fostering reasoning strategies that are both effective and efficient.




\subsection{Boosting Efficient Reasoning in Policy Optimization}
\label{sec:bingo}

We propose a novel reinforcement learning algorithm, \textsc{Bingo} (\textbf{B}oosting Efficient Reason\textbf{ING} in Policy \textbf{O}ptimization), designed to jointly optimize reasoning performance and efficiency. \textsc{Bingo} extends the reinforcement learning framework, primarily based on Proximal Policy Optimization (PPO) in this work, by introducing two key innovations: a \textbf{\textit{significance-aware length reward}} and a \textbf{\textit{dynamic length reward}}.

As discussed in Section~\ref{sec:significance}, we begin by categorizing tokens into \emph{significant} and \emph{insignificant} based on their significance scores. To promote concise yet informative responses, we introduce a cosine-based reward function that adjusts penalties according to the length composition of the response. For correctly answered samples, the reward penalizes only the length of the insignificant portion, reducing verbosity while preserving essential reasoning. For incorrect samples, the reward both penalizes the use of insignificant tokens and encourages the generation of more significant reasoning content.

To balance exploration and efficiency over the course of training, we incorporate a time-dependent mechanism that gradually reduces the incentive for longer responses. As detailed in Section~\ref{sec:dynamic}, this dynamic reward decays as the model converges, shifting the focus from exploration to conciseness.

The overall reward $R^\textsc{Bingo}(y)$ formulation integrates these components into a unified objective:
\begin{equation}
\scriptsize
\begin{aligned}
&R^\textsc{Bingo}(y) = \\
&\begin{cases}
\underbrace{\lambda_c \cdot r_{is}(y)}_{\text{\textbf{\textit{\tiny Correct insignificant length}}}}, 
& \text{if correct}, \\[3ex]
\underbrace{\lambda_w^{is} \cdot \left[r_{is}(y)-1\right]}_{\text{\textbf{\textit{\tiny Incorrect insignificant length}}}} 
+ \quad
\underbrace{\min\left(0,\, r_s(y) - \lambda_w^s \right)}_{\text{\textbf{\textit{\tiny Dynamic significant length}}}}, 
& \text{if incorrect}.
\end{cases}
\end{aligned}
\label{eq:unified-reward}
\end{equation}
where the coefficient \( \lambda_c \) controls the strength of the penalty applied to correct responses, while \( \lambda_w^{is} \) determines the magnitude of the penalty for incorrect ones. The parameter \( \lambda_w^{s} \) serves as a dynamic threshold to balance exploration when the model generates incorrect outputs.


We optimize the policy using the proximal policy optimization objective with the reward \( R_\textsc{Bingo} \) defined by Equation~\ref{eq:unified-reward}. The surrogate objective is:

\begin{equation}
\scriptsize
\mathcal{J}_{\textsc{\tiny Bingo}}(\theta) =
\mathbb{E}_t \left[
\min\left(
r_t(\theta) \hat{A}_t,\;
\operatorname{clip}\left(r_t(\theta), 1 - \epsilon, 1 + \epsilon\right) \hat{A}_t
\right)
\right],
\label{eq:ppo-objective}
\end{equation}

\noindent where:
\begin{itemize}
    \item \( r_t(\theta) = \frac{\pi_\theta(a_t \mid s_t)}{\pi_{\theta_{\text{old}}}(a_t \mid s_t)} \) is the importance sampling ratio,
    \item \( \hat{A}_t \) is the advantage estimate at time step \( t \), computed via generalized advantage estimation using the final sequence-level reward \( R^\textsc{Bingo} \) and the value predictions \( V(s_t) \).
    \item \( \epsilon \) is a clipping parameter.
\end{itemize}

Therefore, \textsc{Bingo} achieves a favorable trade-off between accuracy and efficiency by maximizing the objective function $\mathcal{J}_{\textsc{Bingo}}(\theta)$.

\begin{table*}[t]
\centering
\caption{\textbf{Comparison of different length-based rewards on reasoning benchmarks.} 
Each method is evaluated using DeepSeek-R1-Distill-Qwen-1.5B as the base model by answer accuracy (Acc, \%), response length (Len), and length-normalized accuracy (L-Acc, \%). The best performance is highlighted in \textcolor{langdarkblue}{dark blue}, and the second-best in \textcolor{langwildblue}{light blue}. }
\resizebox{\textwidth}{!}{%
\begin{tabular}{llcccccccccccc}
\toprule
\multicolumn{2}{l}{\textbf{Length-based Reward}} 
& \multicolumn{3}{c}{\textbf{MATH500}} 
& \multicolumn{3}{c}{\textbf{GSM8K}} 
& \multicolumn{3}{c}{\textbf{TheoremQA}} 
& \multicolumn{3}{c}{\textbf{AIME2024}} \\
\cmidrule(lr){3-5} \cmidrule(lr){6-8} \cmidrule(lr){9-11} \cmidrule(lr){12-14}
& & Acc$\uparrow$ & Len$\downarrow$ & L-Acc$\uparrow$ 
    & Acc$\uparrow$ & Len$\downarrow$ & L-Acc$\uparrow$ 
    & Acc$\uparrow$ & Len$\downarrow$ & L-Acc$\uparrow$ 
    & Acc$\uparrow$ & Len$\downarrow$ & L-Acc$\uparrow$ \\
\midrule
& Vanilla PPO \citep{Schulman2017ProximalPO}      & 81.4 & 2,771 & 66.2 & 85.4 & 1,310 & 78.2 & 32.3 & 4,146 & 22.7 & 26.7 & 6,961 & 10.3 \\
& O1-Pruner \citep{luo2025o1prunerlengthharmonizingfinetuningo1like}        & 74.4 &   991 & 69.8 & 81.4 &   \textcolor{langdarkblue}{\textbf{211}} & 80.3 & 28.6 &   \textcolor{langdarkblue}{\textbf{524}} & 27.6 & 26.7 & 5,958 & 13.9 \\
& kimi-k1.5 \citep{kimiteam2025kimik15scalingreinforcement}        & 80.4 & 1,692 & 71.6 & 85.4 &   739 & 81.5 & 34.4 & 2,136 & 29.6 & 33.3 & 5,159 & 20.3 \\
& Effi. Reasoning \citep{arora2025traininglanguagemodels} 
    & \textcolor{langdarkblue}{\textbf{82.6}} 
    & 2,395 
    & 69.5 
    & 86.4 
    & 1,155 
    & 80.0 
    & 34.8 
    & 3,560 
    & 26.2 
    & \textcolor{langdarkblue}{\textbf{36.7}} 
    & 5,771 
    & 19.9 \\
& Demystifying \citep{yeo2025demystifyinglongchainofthoughtreasoning}     
    & 80.2 
    & 1,411 
    & 73.0 
    & 86.6 
    &   548 
    & 83.6
    & 35.1 
    & 1,976 
    & 30.6 
    & 30.0 
    & 6,183 
    & 14.9 \\
& DAST \citep{shi2025dastdifficultyadaptiveslowthinking}             
    & 81.2 
    & 1,770 
    & 71.9 
    & 82.0 
    &   456 
    & 79.6 
    & 35.2 
    & 2,325 
    & 29.8 
    & \textcolor{langdarkblue}{\textbf{36.7}} 
    & 5,400 
    & 21.4 \\
\midrule
\multicolumn{2}{l}{\textbf{\textit{Bingo (Ours)}}} & & & & & & & & & & & & \\
& \textbf{Bingo-A} 
    & \textcolor{langwildblue}{\textbf{82.2}} 
    &  \textcolor{langwildblue}{\textbf{894}} 
    & \textcolor{langdarkblue}{\textbf{77.6}} 
    & \textcolor{langdarkblue}{\textbf{87.0}} 
    &   563 
    & \textcolor{langwildblue}{\textbf{83.9}} 
    & \textcolor{langdarkblue}{\textbf{36.8}} 
    & {1,648} 
    & \textcolor{langwildblue}{\textbf{32.9}} 
    & 33.3 
    & \textcolor{langdarkblue}{\textbf{2,943}} 
    & \textcolor{langdarkblue}{\textbf{26.7}}\\
& \textbf{Bingo-E} 
    & 80.6 
    &   \textcolor{langdarkblue}{\textbf{779}} 
    & \textcolor{langwildblue}{\textbf{76.7}} 
    & \textcolor{langwildblue}{\textbf{86.7}} 
    &   \textcolor{langwildblue}{\textbf{345}} 
    & \textcolor{langdarkblue}{\textbf{84.9}} 
    & \textcolor{langwildblue}{\textbf{36.7}} 
    & \textcolor{langwildblue}{\textbf{1,584}} 
    & \textcolor{langdarkblue}{\textbf{33.0}} 
    & 33.3
    & \textcolor{langdarkblue}{\textbf{2,943}} 
    & \textcolor{langdarkblue}{\textbf{26.7}} \\
\bottomrule
\end{tabular}%
}
\label{tab:baseline}
\vspace{-0.7em}
\end{table*}

\begin{table*}[t]
\centering
\setlength{\tabcolsep}{4pt}
\caption{\textbf{Performance comparison across model scales and types.}
Accuracy (Acc), average output length (Length), and length-normalized accuracy (L-Acc, \%) on four benchmarks.
The best performance is highlighted in \textcolor{langdarkblue}{dark blue}, and the second-best in \textcolor{langwildblue}{light blue}.}
\resizebox{0.9\textwidth}{!}{%
\begin{tabular}{lcccccccccccc}
\toprule
\multicolumn{1}{l}{\textbf{Method}}
  & \multicolumn{3}{c}{\textbf{MATH500}}
  & \multicolumn{3}{c}{\textbf{GSM8K}}
  & \multicolumn{3}{c}{\textbf{TheoremQA}}
  & \multicolumn{3}{c}{\textbf{AIME2024}} \\
\cmidrule(lr){2-4}\cmidrule(lr){5-7}\cmidrule(lr){8-10}\cmidrule(lr){11-13}
  & Acc$\uparrow$ & Len$\downarrow$ & L-Acc$\uparrow$
  & Acc$\uparrow$ & Len$\downarrow$ & L-Acc$\uparrow$
  & Acc$\uparrow$ & Len$\downarrow$ & L-Acc$\uparrow$
  & Acc$\uparrow$ & Len$\downarrow$ & L-Acc$\uparrow$ \\
\midrule
\multicolumn{1}{l}{\textit{DeepSeek-R1-Distill-Qwen-1.5B}} \\
\quad Base           & 63.2 & 3,913 & 45.7 & 73.2 & 2,025 & 63.5 & 18.7 & 5,741 & 10.3 & 16.7 & 7,027 & 6.3 \\
\quad PPO            & \textcolor{langwildblue}{81.4} & 2,771 & 66.2 & 85.4 & 1,310 & 78.2 & 32.3 & 4,146 & 22.7 & 26.7 & 6,961 & 10.3 \\
\quad Bingo-A (Ours) & \textcolor{langdarkblue}{82.2} & \textcolor{langwildblue}{894} & \textcolor{langdarkblue}{77.6}
                     & \textcolor{langdarkblue}{87.0} & \textcolor{langwildblue}{563} & \textcolor{langwildblue}{83.9}
                     & \textcolor{langdarkblue}{36.8} & \textcolor{langwildblue}{1,648} & \textcolor{langwildblue}{32.9}
                     & \textcolor{langdarkblue}{33.3} & \textcolor{langdarkblue}{2,943} & \textcolor{langdarkblue}{26.7} \\
\quad Bingo-E (Ours) & 80.6 & \textcolor{langdarkblue}{779} & \textcolor{langwildblue}{76.7}
                     & \textcolor{langwildblue}{86.7} & \textcolor{langdarkblue}{345} & \textcolor{langdarkblue}{84.9}
                     & \textcolor{langwildblue}{36.7} & \textcolor{langdarkblue}{1,584} & \textcolor{langdarkblue}{33.0}
                     & \textcolor{langdarkblue}{33.3} & \textcolor{langdarkblue}{2,943} & \textcolor{langdarkblue}{26.7} \\
\midrule
\multicolumn{1}{l}{\textit{DeepSeek-R1-Distill-Qwen-7B}} \\
\quad Base           & 82.8 & 3,033 & 65.7 & 85.7 & 1,001 & 80.3 & 37.8 & 4,340 & 25.9 & 40.0 & 6,528 & 18.0 \\
\quad PPO            & \textcolor{langwildblue}{88.4} & 1,536 & \textcolor{langwildblue}{79.7}
                     & \textcolor{langdarkblue}{92.9} & 918 & 87.5
                     & \textcolor{langdarkblue}{45.4} & 2,709 & 37.1 & 56.7 & 5,857 & 30.3 \\
\quad Bingo-A (Ours) & \textcolor{langdarkblue}{88.8} & \textcolor{langwildblue}{1,400} & \textcolor{langdarkblue}{80.9}
                     & \textcolor{langwildblue}{92.3} & \textcolor{langwildblue}{371} & \textcolor{langdarkblue}{90.2}
                     & \textcolor{langwildblue}{45.2} & \textcolor{langwildblue}{1,908} & \textcolor{langwildblue}{39.6}
                     & \textcolor{langdarkblue}{63.3} & \textcolor{langwildblue}{4,670} & \textcolor{langwildblue}{41.5} \\
\quad Bingo-E (Ours) & 87.2 & \textcolor{langdarkblue}{1,252} & 80.3
                     & 91.8 & \textcolor{langdarkblue}{366} & \textcolor{langwildblue}{89.7}
                     & 45.0 & \textcolor{langdarkblue}{1,693} & \textcolor{langdarkblue}{40.1}
                     & \textcolor{langwildblue}{60.0} & \textcolor{langdarkblue}{4,011} & \textcolor{langdarkblue}{42.9} \\
\midrule
\multicolumn{1}{l}{\textit{Qwen2.5-Math-7B-Instruct}} \\
\quad Base           & 80.8 & 727 & 70.3 & 95.8 & 331 & 90.3 & 36.8 & 919 & 30.7 & 16.7 & 1,310 & 12.5 \\
\quad PPO            & \textcolor{langwildblue}{82.0} & 670 & 72.3
                     & \textcolor{langdarkblue}{96.6} & 305 & \textcolor{langwildblue}{91.5}
                     & \textcolor{langwildblue}{37.6} & 759 & 32.5
                     & \textcolor{langdarkblue}{20.0} & 1,260 & \textcolor{langwildblue}{15.2} \\
\quad Bingo-A (Ours) & \textcolor{langdarkblue}{82.6} & \textcolor{langwildblue}{656} & \textcolor{langwildblue}{73.0}
                     & \textcolor{langwildblue}{96.1} & \textcolor{langwildblue}{283} & \textcolor{langwildblue}{91.5}
                     & \textcolor{langdarkblue}{37.9} & \textcolor{langwildblue}{598} & \textcolor{langdarkblue}{34.0}
                     & \textcolor{langdarkblue}{20.0} & \textcolor{langwildblue}{892} & \textcolor{langdarkblue}{16.8} \\
\quad Bingo-E (Ours) & 81.6 & \textcolor{langdarkblue}{559} & \textcolor{langdarkblue}{73.6}
                     & 96.0 & \textcolor{langdarkblue}{241} & \textcolor{langdarkblue}{92.0}
                     & 37.1 & \textcolor{langdarkblue}{552} & \textcolor{langwildblue}{33.5}
                     & 16.7 & \textcolor{langdarkblue}{811} & 14.2 \\
\bottomrule
\end{tabular}%
}
\label{tab:model_scale}
\vspace{-0.8em}
\end{table*}

\section{Experiments}

\subsection{Experimental Setup and Evaluation Metrics}
\label{sec:exp}

We fine-tune two reasoning models, \textit{DeepSeek-R1-Distill-Qwen-1.5B} and \textit{DeepSeek-R1-Distill-Qwen-7B}, together with an instruction-tuned model, \textit{Qwen2.5-Math-7B-Instruct}, on the \textsc{MATH} training split. Evaluation is performed on \textsc{MATH500}, \textsc{GSM8K}, \textsc{AIME2024}, and \textsc{TheoremQA}. \textsc{MATH500} serves as the in-distribution benchmark, while the remaining datasets constitute out-of-distribution evaluations, covering both mathematical reasoning, and commonsense question answering.

We compare \textsc{Bingo} against a frozen base model, a Vanilla PPO baseline, and several recent state-of-the-art methods. To enable fair comparison, all baselines are re-implemented within a unified PPO framework, isolating the effect of different length-based reward designs. Detailed experimental settings and baseline descriptions are provided in Appendix~\ref{ap:detail_setting}. A thorough hyperparameter study can be found in Appendix~\ref{ap:hyperparameter}.

We report two model variants: \textsc{Bingo-A}, selected for peak validation accuracy, and \textsc{Bingo-E}, selected for stable response length. This dual-reporting strategy enables practitioners to choose a model variant based on their preference for accuracy or efficiency. They may correspond to the same checkpoint. To evaluate reasoning efficiency, we report accuracy (Acc), response length (Len), and length-normalized accuracy (L-Acc), which jointly measures correctness and conciseness:
\begin{equation}
  \mathrm{L\text{-}Acc}=
  \mathrm{Acc}\times\sqrt{1-\frac{L}{L_{\max}}}.
\end{equation}
where $L$ is the average response length and $L_{\max}$ is the maximum allowable length. A detailed definition and theoretical analysis of L-Acc are provided in Appendix~\ref{ap:l-acc}.

\subsection{Performance Comparison with Baseline Methods}

\noindent \textbf{\textsc{Bingo} outperforms existing methods in L-Acc:} As shown in Table~\ref{tab:baseline}, both \textsc{Bingo-A} and \textsc{Bingo-E} achieve the highest L-Acc across all four benchmarks, outperforming previous baselines such as Vanilla PPO, Efficient Reasoning, and DAST.

\noindent \textbf{\textsc{Bingo-A} improves accuracy while significantly reducing response length:} \textsc{Bingo-A} reduces average response length by up to 68\% compared to Vanilla PPO (e.g., 894 vs. 2,771 tokens on \textsc{MATH500}), demonstrating the model's ability to generate concise and correct reasoning steps.

\noindent \textbf{Existing baselines struggle with the trade-off between accuracy and brevity:} Approaches like Efficient Reasoning produce verbose outputs, while methods such as O1-Pruner overly shorten responses, compromising accuracy.

\subsection{Performance Evaluation across Varying Model Scales}

\noindent \textbf{\textsc{Bingo} achieves the best trade-off between accuracy and response length across different model sizes:} As shown in Table~\ref{tab:model_scale}, both \textsc{Bingo-A} and \textsc{Bingo-E} consistently outperform all other methods across various model sizes (1.5B and 7B parameters) and benchmarks, achieving the highest L-Acc while maintaining competitive or superior accuracy.

\noindent \textbf{\textsc{Bingo-E} offers a substantial reduction in response length without sacrificing accuracy:} \textsc{Bingo-E} reduces response length by up to 63\% (e.g., 366 vs. 1,001 tokens on \textsc{GSM8K}) compared to the Base model, while also improving accuracy by 6.1 percentage points, demonstrating the model's ability to generate concise and accurate reasoning steps.

\subsection{Analysis of Significant versus Insignificant Token Ratio}

\begin{figure*}[t]
    \centering
    \includegraphics[width=0.9\textwidth]{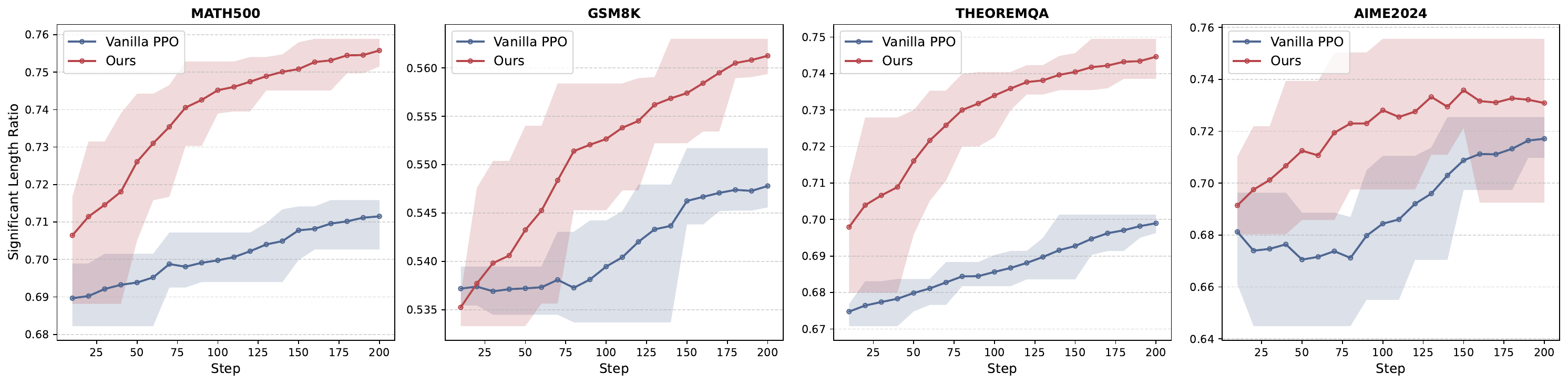}
    \caption{\textbf{Significant Length Ratio dynamics during training.} The x-axis indicates training steps, and the y-axis denotes the proportion of significant tokens in the generated responses. 
Each subplot corresponds to one benchmark evaluated using DeepSeek-R1-Distill-Qwen-1.5B as the base model. 
The blue curve represents the baseline method (Vanilla PPO), and the red curve represents our approach (Ours). 
}
    \label{fig:sig_length}
    \vspace{-0.7em}
\end{figure*}

Our significance-aware reward consistently increases the proportion of significant tokens across all datasets (e.g., from 0.71 to 0.75 on MATH500), leading to shorter and more focused reasoning chains, as illustrated in Figure~\ref{fig:sig_length}. By reducing redundant tokens, this results in lower response length and improved Acc and L-Acc, demonstrating enhanced reasoning efficiency without sacrificing essential content.

\begin{table*}[t]
\centering
\caption{\textbf{Ablation study on reward components.} 
Each method is evaluated using DeepSeek-R1-Distill-Qwen-1.5B as the base model by answer accuracy (Acc, \%), response length (Len), and length-normalized accuracy (L-Acc, \%).
Values in parentheses indicate the relative drop in L-Acc compared to \textsc{Bingo-A}.
The best performance is highlighted in \textcolor{langdarkblue}{dark blue}, and the second-best in \textcolor{langwildblue}{light blue}.}
\resizebox{\textwidth}{!}{%
\begin{tabular}{lccccccccccccc}
\toprule
\multicolumn{1}{l}{\textbf{Method}} 
  & \multicolumn{3}{c}{\textbf{MATH500}}
  & \multicolumn{3}{c}{\textbf{GSM8K}}
  & \multicolumn{3}{c}{\textbf{TheoremQA}}
  & \multicolumn{3}{c}{\textbf{AIME2024}} \\
\cmidrule(lr){2-4} \cmidrule(lr){5-7} \cmidrule(lr){8-10} \cmidrule(lr){11-13}
  & Acc$\uparrow$ & Len$\downarrow$ & L-Acc$\uparrow$
  & Acc$\uparrow$ & Len$\downarrow$ & L-Acc$\uparrow$
  & Acc$\uparrow$ & Len$\downarrow$ & L-Acc$\uparrow$
  & Acc$\uparrow$ & Len$\downarrow$ & L-Acc$\uparrow$ \\
\midrule
\textbf{Bingo-A (Ours)}
  & \textcolor{langdarkblue}{\textbf{82.2}} & \textcolor{langdarkblue}{\textbf{894}} & \textcolor{langdarkblue}{\textbf{77.6}}
  & \textcolor{langdarkblue}{\textbf{87.0}} & 563 & \textcolor{langdarkblue}{\textbf{83.9}}
  & \textcolor{langwildblue}{\textbf{36.8}} & \textcolor{langwildblue}{\textbf{1,648}} & \textcolor{langdarkblue}{\textbf{32.9}}
  & \textcolor{langwildblue}{\textbf{33.3}} & \textcolor{langdarkblue}{\textbf{2,943}} & \textcolor{langdarkblue}{\textbf{26.7}} \\
\midrule
Vanilla PPO
  & \textcolor{langwildblue}{\textbf{81.4}} & 2,771 & 66.2 \textcolor{gray}{(-14.7)}
  & 85.4 & 1,310 & 78.2 \textcolor{gray}{(-6.8)}
  & 32.3 & 4,146 & 22.7 \textcolor{gray}{(-31.0)}
  & 26.7 & 6,961 & 10.3 \textcolor{gray}{(-16.4)} \\
Significance-Aware Length Reward
  & \textcolor{langwildblue}{\textbf{81.4}} & 1,734 & \textcolor{langwildblue}{\textbf{72.3}} \textcolor{gray}{(-5.3)}
  & \textcolor{langwildblue}{\textbf{86.7}} & 742 & 82.6 \textcolor{gray}{(-1.3)}
  & 36.0 & 2,841 & 29.1 \textcolor{gray}{(-3.8)}
  & \textcolor{langdarkblue}{\textbf{40.0}} & 5,138 & \textcolor{langwildblue}{\textbf{24.4}} \textcolor{gray}{(-2.3)} \\
\quad w/o Cosine
  & 78.6 & 1,750 & 69.7 \textcolor{gray}{(-7.9)}
  & 85.7 & \textcolor{langwildblue}{\textbf{509}} & 83.0 \textcolor{gray}{(-0.9)}
  & 35.3 & 2,414 & 29.7 \textcolor{gray}{(-3.2)}
  & \textcolor{langwildblue}{\textbf{33.3}} & 6,454 & 15.4 \textcolor{gray}{(-11.3)} \\
\quad w/o Significance Separation
  & 79.8 & 1,666 & 71.2 \textcolor{gray}{(-6.4)}
  & 86.6 & 604 & \textcolor{langwildblue}{\textbf{83.3}} \textcolor{gray}{(-0.6)}
  & \textcolor{langdarkblue}{\textbf{36.9}} & 2,328 & 31.3 \textcolor{gray}{(-1.6)}
  & 26.7 & 5,702 & 14.7 \textcolor{gray}{(-12.0)} \\
\quad w/o Length Incentive
  & 77.8 & \textcolor{langwildblue}{\textbf{1,400}} & 70.8 \textcolor{gray}{(-6.8)}
  & 82.6 & \textcolor{langdarkblue}{\textbf{425}} & 80.5 \textcolor{gray}{(-3.4)}
  & 35.7 & \textcolor{langdarkblue}{\textbf{1,636}} & \textcolor{langwildblue}{\textbf{32.0}} \textcolor{gray}{(-0.9)}
  & 30.0 & \textcolor{langwildblue}{\textbf{4,157}} & 21.1 \textcolor{gray}{(-5.6)} \\
\midrule
Dynamic Length Reward
  & 79.0 & 2,204 & 67.5 \textcolor{gray}{(-10.1)}
  & 84.3 & 955 & 79.2 \textcolor{gray}{(-4.7)}
  & 33.9 & 2,632 & 27.9 \textcolor{gray}{(-5.0)}
  & 30.0 & 5,047 & 18.6 \textcolor{gray}{(-8.1)} \\
\bottomrule
\end{tabular}%
}
\label{tab:ablation}
\vspace{-1em}
\end{table*}

\subsection{Ablation Study}


\noindent \textbf{Combining Significance-Aware and Dynamic Length rewards yields the best trade-off:} Table \ref{tab:ablation} shows that the joint use of both the Significance-Aware and Dynamic Length rewards (\textsc{Bingo-A}) provides the best performance, achieving the highest accuracy and L-Acc across all four benchmarks, while maintaining competitive or superior raw accuracy compared to other methods.

\noindent \textbf{Removing key reward components degrades performance significantly:} Ablations show that removing any of the reward components leads to noticeable performance drops, particularly in terms of L-Acc. 


\subsection{Additional Experiments and Analysis}

\noindent \textbf{\textsc{Bingo} improves across multiple RL algorithms:} We evaluate the generalizability of our reward design by integrating it into other RL algorithms, including RLOO, GRPO, and Reinforce++. As shown in Appendix~\ref{ap:rl}, \textsc{Bingo} variants consistently outperform vanilla ones, achieving superior performance in both accuracy and L-Acc.

\noindent \textbf{\textsc{Bingo} effectively reduces response length, especially for incorrect samples:} The distribution of response lengths for correct vs. incorrect samples in Appendix~\ref{ap:response_c_w} shows that \textsc{Bingo} significantly shortens incorrect sample lengths compared to PPO. Furthermore, Figure \ref{fig:correct_wrong} in Appendix~\ref{ap:response_c_w} illustrates that incorrect samples show a more significant reduction in response length during training, confirming the dynamic reward's effectiveness. Figure \ref{fig:length} in Appendix~\ref{ap:response_trends} further shows that our method consistently reduces response length more than PPO.

\noindent \textbf{Analysis of significant token ratio and token-level significance visualization:} Appendix~\ref{ap:sig_token_ratio} shows that \textsc{Bingo} increases the proportion of significant tokens compared to baselines, while Appendix~\ref{ap:sig_visualization} provides a token-level significance visualization, demonstrating how our approach retains essential reasoning steps and eliminates redundancy.

\noindent \textbf{Analysis confirms the effectiveness of dynamic and significance rewards:} The analysis in Appendix~\ref{ap:response_reward} validates that our dynamic and significance rewards balance exploration and efficiency. A case study in Appendix~\ref{ap:case_study} further demonstrates the practical impact of \textsc{Bingo} on reasoning efficiency.

\section{Conclusion}
In this work, we revisit reasoning efficiency in LLMs from a token-level perspective and propose \textsc{Bingo}, an RL framework that jointly optimizes accuracy and efficiency. By incorporating token significance into reward design, \textsc{Bingo} selectively penalizes insignificant tokens while preserving essential reasoning, and uses a dynamic length reward to balance exploration and conciseness during training, achieving a favorable trade-off between correctness and efficiency.

\section*{Limitations}
Despite the promising results achieved by \textsc{Bingo}, several aspects warrant further exploration. First, our current study instantiates \textsc{Bingo} within a PPO-based training framework, and we leave a more comprehensive investigation of its compatibility with other reinforcement learning algorithms to future work. While the underlying design is general, different optimization schemes may require tailored adaptations to fully realize its potential. Second, since \textsc{Bingo} follows standard PPO-style optimization, its training cost is similar to that of PPO. Future work may explore more efficient training strategies, such as curriculum learning or partial parameter updates, to further improve scalability.

\section*{Acknowledgments}
Life, much like the realm of reasoning where complexity often clouds clarity, is filled with choices. Each decision, each token, carries with it uncertainty, hesitation, and fear. Yet amidst this ambiguity, we seek something deeper: to make each step joyful and meaningful, much like selecting only the most significant tokens. \textsc{Bingo} is our call to action, a declaration that precision and progress emerge not from waiting but from doing. We must summon the courage to take that first bold step, embracing uncertainty without fear and moving forward without hesitation. In doing so, we choose joy, and we choose meaning. As we journey onward, token by token and challenge by challenge, we find that with each thoughtful step, the path unfolds. And in the end, the answer becomes clear, resounding unmistakably: \textit{\textbf{Bingo}}. Thanks for everything throughout this journey.

\bibliography{custom}

\clearpage
\DoToC
\clearpage

\appendix


\section{Ethical Considerations}
Our research focuses on improving the efficiency of large language model reasoning through reinforcement learning techniques, which poses no direct ethical concerns regarding human subjects, as no human data collection or experimentation was conducted. All datasets used (MATH, GSM8K, TheoremQA, AIME2024) are publicly available benchmarks with proper citations. We acknowledge the broader implications of more efficient LLM reasoning, including potential dual-use concerns, but emphasize that our contributions aim to reduce computational costs and environmental impact of AI systems. The research was conducted with academic integrity, and all authors have reviewed and agree with the content presented. There are no conflicts of interest to declare.

\section{Reproducibility Statement}
To ensure reproducibility, we provide comprehensive implementation details throughout the paper. Section~\ref{sec:method} describes our complete algorithmic framework, including all hyperparameters and reward formulations. Section~\ref{sec:exp} details our experimental setup, while Appendix~\ref{ap:detail_setting} provides comprehensive settings including training and evaluation configurations, dataset settings, optimization parameters, data splits, evaluation metrics, and computational requirements. All experiments use publicly accessible pre-trained models (DeepSeek-R1-Distill-Qwen-1.5B, DeepSeek-R1-Distill-Qwen-7B, and Qwen2.5-Math-7B-Instruct) and datasets available on HuggingFace. Our code is available through an anonymous link in the abstract and as a zip file in the supplemental materials, and will be made publicly available upon acceptance.

\section{Discussion of Token Significance Measurement}
\label{ap:sign_measure}

As introduced in Section~\ref{sec:task}, token significance is computed using a contextual scoring function. In our work, this function is instantiated with \textit{LLMLingua-2}~\citep{pan2024llmlingua2datadistillationefficient}. Given a generated sequence
\( y = (y_1, y_2, \dots, y_n) \),
each token \( y_i \) is assigned a significance score
\begin{equation}
S(y_i) = P(y_i \mid \mathbf{y}_{\le n}; \boldsymbol{\theta}_{\mathcal{M}_e}),
\end{equation}
where \( \boldsymbol{\theta}_{\mathcal{M}_e} \) denotes the parameters of a bidirectional encoder model. By conditioning on both preceding and succeeding context, this formulation enables a balanced assessment of token-level contribution within the full sequence and mitigates positional effects inherent to left-to-right estimation. In practice, \textit{LLMLingua-2} provides a lightweight and computationally efficient instantiation of this contextual scoring function.

To further assess whether \textit{LLMLingua-2} provides reliable token significance estimates for our setting, we manually annotated 100 representative samples and compared human labels against \textit{LLMLingua-2}-generated labels. In the manual annotation, tokens labeled as \emph{significant} mainly included key derivation steps, numerical values, formulas, and essential action verbs, while tokens labeled as \emph{insignificant} mainly included common linking words, fillers, and conjunctions such as ``and'' and ``or.'' The comparison results are reported in Table~\ref{tab:llmlingua2_manual_eval}. \textit{LLMLingua-2} achieves an average precision of 0.8026, an average recall of 0.9222, and an average F1-score of 0.8518. In particular, the high recall indicates that \textit{LLMLingua-2} successfully captures most tokens that humans identify as significant. Although minor discrepancies remain, which is expected given linguistic ambiguity and annotation subjectivity, the overall results suggest that \textit{LLMLingua-2} provides sufficiently accurate and robust token significance estimates for our purpose.

\begin{table}[t]
\centering
\small
\caption{Comparison between manual token significance annotations and \textit{LLMLingua-2}-generated labels on 100 representative samples.}
\begin{tabular}{lc}
\toprule
Metric & Score \\
\midrule
Average Precision & 0.8026 \\
Average Recall & 0.9222 \\
Average F1-score & 0.8518 \\
\bottomrule
\end{tabular}
\label{tab:llmlingua2_manual_eval}
\end{table}

An alternative approach is \textit{Selective Context}~\citep{li2023compressingcontextenhanceinference}, which defines token importance using a unidirectional language modeling objective:
\begin{equation}
S(y_i) = -\log P(y_i \mid y_{<i}; \theta_{M_L}),
\end{equation}
where \( \theta_{M_L} \) denotes the parameters of a unidirectional language model. Under this formulation, token importance is interpreted through model uncertainty, with less predictable tokens receiving higher scores. However, because it relies exclusively on past context, this approach is prone to position bias and cannot capture interactions between a token and subsequent reasoning steps, thereby limiting its ability to assess token-level contribution within the complete reasoning trace.

In contrast to such unidirectional uncertainty-based measures, contextual scoring based on bidirectional encoders offers a more holistic view of token significance. By evaluating tokens with respect to the entire sequence, it better aligns with the objective of discouraging redundant or low-contribution tokens while preserving essential reasoning content. The manual evaluation above further supports that \textit{LLMLingua-2} can identify human-recognized significant tokens with high fidelity, making it a suitable choice for our token significance measurement. The effectiveness of this type of contextual, bidirectional scoring has also been supported by recent work such as TokenSkip~\citep{xia2025tokenskipcontrollablechainofthoughtcompression}, which demonstrates its ability to identify redundant components in chain-of-thought reasoning.

\section{Theoretical Analysis of Significance-Aware Length Reward}
\label{ap:theory}

\vspace{0.5\baselineskip}
\noindent\textbf{Preliminaries and Notation.}  
Given a prompt $x$, the policy $\pi_\theta$ generates a chain-of-thought
(CoT) sequence $Y = (y_1,\dots,y_T)$, from which a deterministic decoder produces a final answer $\hat{Z} \in \mathcal{Z}$.  
To assess the relative informativeness of each token, we compute a \emph{significance score} using \textit{LLMLingua-2}~\citep{pan2024llmlingua2datadistillationefficient}:

\begin{equation}
S(y_i) = P(y_i \mid Y;\,\boldsymbol{\theta}_{\mathcal{M}_e}),
\end{equation}

where $\boldsymbol{\theta}_{\mathcal{M}_e}$ denotes the parameters of a bidirectional encoder model.  
Unlike unidirectional predictors, \textit{LLMLingua-2} uses both left and right context to provide a holistic estimate of token informativeness.  
Based on a fixed threshold $\tau$, we partition the sequence as:

\begin{gather}
\mathcal{Y}_{\mathrm{sig}} = \{y_i : S(y_i) \ge \tau\},\\
\mathcal{Y}_{\mathrm{insig}} = \{y_i : S(y_i) < \tau\},
\end{gather}

where $\mathcal{Y}_{\mathrm{sig}}$ and $\mathcal{Y}_{\mathrm{insig}}$ denote the sets of \textit{significant} and \textit{insignificant} tokens, respectively.

\medskip
\noindent\textbf{Motivation for a Mutual Information Proxy.}  
In principle, each token’s importance could be measured by its mutual information with the final answer, $I(y_i;Z^\star)$.  
However, computing the exact joint distribution $p(y_i, Z^\star)$ is intractable due to the vast generation space and limited supervision.  
Instead, we employ a proxy that is
(i) efficient to compute for each token and
(ii) monotonically correlated with $I(y_i;Z^\star)$.

\textit{LLMLingua-2} \citep{pan2024llmlingua2datadistillationefficient} satisfies these requirements by training under an information bottleneck objective:
\begin{equation}
I(T;Y) - \beta\,I(T;Z^\star),
\end{equation}
where $T$ is the retained subsequence. Tokens with low conditional probability typically carry little additional information about $Z^\star$, while high-probability tokens preserve essential semantics.

\medskip
\begin{assumption}[Fidelity of the Mutual Information Proxy]
\label{as:mi-proxy}
There exist constants $c > \varepsilon > 0$ such that
\begin{gather}
I(y_i;Z^\star)\le\varepsilon
\quad(\forall\,y_i\in\mathcal{Y}_{\mathrm{insig}}), \\
I(y_j;Z^\star)\ge c-\varepsilon
\quad(\forall\,y_j\in\mathcal{Y}_{\mathrm{sig}}).
\end{gather}
\end{assumption}

\medskip
\begin{lemma}[Bounded Accuracy Loss]
\label{lem:acc-loss}
Let $\hat{Z}_{\text{full}}$ denote the answer decoded from the full CoT, and
$\hat{Z}_{\text{sig}}$ the answer decoded after removing
$\mathcal{Y}_{\mathrm{insig}}$.  
Under Assumption~\ref{as:mi-proxy}, the increase in error probability is bounded:
\begin{equation}
\Bigl|
\Pr[\hat{Z}_{\text{sig}}\neq Z^\star]
-
\Pr[\hat{Z}_{\text{full}}\neq Z^\star]
\Bigr|
\;\le\;\varepsilon.
\end{equation}
\end{lemma}

\begin{proof}
Let $Y=(y_1,\dots,y_T)$ and
$Y_{\!\mathrm{sig}} = Y \setminus \mathcal{Y}_{\mathrm{insig}}$.
By the chain rule:
\begin{equation}
I(Y;Z^\star) = I(Y_{\!\mathrm{sig}};Z^\star) +
         \sum_{y_i\in\mathcal{Y}_{\mathrm{insig}}}
         I(y_i;Z^\star\mid Y_{<i}),
\end{equation}
and each term in the sum is at most $\varepsilon$. Therefore,
\begin{equation}
I(Y;Z^\star) - I(Y_{\!\mathrm{sig}};Z^\star) \le T\varepsilon,
\end{equation}
and by Fano’s inequality, this gap translates into an error increase of at most $\varepsilon$.
\end{proof}

\medskip
\begin{definition}[General vs.\ Significance-Aware Length Reward]
\label{def:rewards}
For a generated trace $Y$, define two reward functions:
\begin{align}
R_{\mathrm{len}}(Y)
&= \mathbbm{1}[\hat{Z}(Y)=Z^\star] - \lambda\,|Y|, \\
R_{\mathrm{sig}}(Y)
&= \mathbbm{1}[\hat{Z}(Y)=Z^\star] - \lambda\,|\mathcal{Y}_{\mathrm{insig}}|.
\end{align}
Here, $R_{\mathrm{len}}$ penalizes total length, while
$R_{\mathrm{sig}}$ penalizes only insignificant tokens.
\end{definition}

\medskip
\begin{theorem}[Benefit of the Significance-Aware Reward]
\label{thm:sig-better}
Let $\pi_\theta$ be updated by a single PPO step using either reward, with the same coefficient $\lambda > 0$.  
If
\begin{equation}
\lambda
>
\frac{\varepsilon}
     {\mathbb{E}_{\pi_\theta}[|\mathcal{Y}_{\mathrm{sig}}|]},
\end{equation}
then
\begin{equation}
\mathbb{E}_{\pi_\theta}[R_{\mathrm{sig}}(Y)]
\;>\;
\mathbb{E}_{\pi_\theta}[R_{\mathrm{len}}(Y)].
\end{equation}
\end{theorem}

\begin{proof}
Lemma~\ref{lem:acc-loss} implies
\begin{equation}
\mathbb{E}[R_{\mathrm{sig}}] - \mathbb{E}[R_{\mathrm{len}}]
=
\lambda\,\mathbb{E}[|\mathcal{Y}_{\mathrm{sig}}|]
- \Delta_{\text{acc}},
\end{equation}
where
\begin{equation}
0 \le \Delta_{\text{acc}} \le \varepsilon.
\end{equation}
Under the stated bound on $\lambda$, the difference is strictly positive.
\end{proof}

\medskip
\noindent\textbf{Practical Implication.}  
The significance-aware reward achieves the same or greater length reduction with provably smaller accuracy degradation than a general length reward. By selectively penalizing insignificant tokens, it still encourages conciseness while maintaining fidelity. With \textit{LLMLingua-2} providing a fast proxy for token–answer informativeness, this reward design supports both principled and practical optimization for efficient reasoning.

\section{Theoretical Discussion of Dynamic Length Reward}
\label{ap:discussion_dynamic}

We provide a theoretical discussion of the motivation for our dynamic length reward schedule by addressing three key questions:
\begin{enumerate}[leftmargin=1.5em]
  \item Why does encouraging longer chains of thought (CoT) during early training help exploration?
  \item Why does applying a fixed length penalty throughout training limit performance?
  \item Why does dynamically flipping the reward from positive to negative upon convergence yield better accuracy–efficiency trade-offs?
\end{enumerate}

\noindent\textbf{1. Longer CoT Enables Richer Exploration.}

Let $P_t(L)$ denote the model’s distribution over output length $L$ at training step $t$.  
Define the expected accuracy given length $L$ as
\begin{equation}
\mathrm{Acc}(L) \;=\; \Pr[\hat{Z}=Z^\star \mid L(Y)=L],
\end{equation}
where $\hat{Z}$ is the predicted answer and $Z^\star$ the ground-truth.  
Empirically, $\mathrm{Acc}(L)$ follows a saturating “S-curve”:
\begin{gather}
\frac{d}{dL}\,\mathrm{Acc}(L) > 0 \quad \text{for } L < L^\star,\\
\frac{d}{dL}\,\mathrm{Acc}(L) \approx 0 \quad \text{for } L \ge L^\star,
\end{gather}
where $L^\star$ is the length at which accuracy saturates.  
The expected accuracy at step $t$ is
\begin{equation}
\mathrm{Acc}_t = \sum_{L} P_t(L)\,\mathrm{Acc}(L).
\end{equation}
Shifting probability mass toward longer CoT (up to $L^\star$) thus increases $\mathrm{Acc}_t$, since longer reasoning expands exploration and raises the chance of discovering correct solution patterns.

\paragraph{Takeaway.} Rewarding longer CoT early boosts exploration and accelerates convergence toward high accuracy.

\bigskip
\noindent\textbf{2. Static Length Penalty Causes Premature Compression.}

Consider a fixed penalty $\lambda>0$, yielding reward
\begin{equation}
J_{\text{static}}(L) = \mathrm{Acc}(L) - \lambda L.
\end{equation}
The optimal length $L_s$ under this objective satisfies
\begin{equation}
\frac{d}{dL}\,\mathrm{Acc}(L)\big|_{L=L_s} = \lambda.
\end{equation}
Since $\frac{d}{dL}\,\mathrm{Acc}(L)$ vanishes for $L \ge L^\star$, any $\lambda>0$ forces $L_s < L^\star$, implying
\begin{equation}
\mathrm{Acc}(L_s) < \mathrm{Acc}(L^\star).
\end{equation}
Thus the model truncates its CoT before accuracy has fully converged.

\paragraph{Takeaway.} Static penalties enforce efficiency too early, sacrificing potential accuracy gains.

\bigskip
\noindent\textbf{3. Dynamic Penalty Supports a Two-Phase Curriculum.}

We introduce a time-dependent penalty $\lambda_t$:
\begin{equation}
\lambda_t =
\begin{cases}
\gamma, & t < t_0 \quad (\text{exploration phase}),\\[4pt]
\alpha\,(t - t_0), & t \ge t_0 \quad (\text{compression phase}),
\end{cases}
\end{equation}
where $\gamma < 0$, and $t_0$ is the step at which validation accuracy stabilizes, i.e.,
\begin{equation}
\Delta \mathrm{Acc}_t = \frac{\mathrm{Acc}_t - \mathrm{Acc}_{t-\Delta}}{\Delta} < \beta.
\end{equation}

\paragraph{Phase I (Exploration).}  
During early training, we set $\lambda_t < 0$, effectively turning the penalty into a bonus:
\[
J(L) = \mathrm{Acc}(L) - \lambda_t L,
\quad \text{with } -\lambda_t > 0,
\]
which encourages longer outputs. Since $\mathrm{Acc}(L)$ increases with $L$ up to $L^\star$, this promotes
\begin{equation}
L_t \to L^\star,
\qquad
\mathrm{Acc}_t \to \mathrm{Acc}(L^\star).
\end{equation}

\paragraph{Phase II (Compression).}  
As training progresses, $\lambda_t$ transitions from negative to positive and increases gradually.  
When $\lambda_t > 0$, the derivative of the reward at $L^\star$ is
\begin{equation}
\begin{split}
\left.\frac{d}{dL}\bigl[\mathrm{Acc}(L)-\lambda_t L\bigr]\right|_{L=L^\star}
= \\\frac{d}{dL}\,\mathrm{Acc}(L^\star) - \lambda_t < 0,
\end{split}
\end{equation}
so extending beyond $L^\star$ reduces reward. The policy thus shortens to a new equilibrium $L_d$:
\begin{gather}
\frac{d}{dL}\,\mathrm{Acc}(L)\big|_{L=L_d} = \lambda_t,
\quad L_d < L^\star,\\
\mathrm{Acc}(L_d)\approx \mathrm{Acc}(L^\star).
\end{gather}

\paragraph{Comparison to Static Penalty.}  
The final dynamic reward is
\begin{equation}
J_{\text{dyn}} = \mathrm{Acc}(L_d) - \lambda_T L_d,
\end{equation}
and under the concavity of $\mathrm{Acc}(L)$ one can show
\begin{equation}
\begin{split}
J_{\text{dyn}} - J_{\text{static}}
= \bigl[\mathrm{Acc}(L_d)-\mathrm{Acc}(L_s)\bigr] \\
- \lambda_T (L_d - L_s) \ge 0,
\end{split}
\end{equation}
i.e., dynamic scheduling yields no worse and often strictly better reward.  
This holds because for concave functions
\begin{equation}
\mathrm{Acc}(L_d) - \mathrm{Acc}(L_s) \;\ge\; \mathrm{Acc}'(L_d) (L_d - L_s),
\end{equation}
and with $\mathrm{Acc}'(L_d) = \lambda_T$, the inequality follows.

\paragraph{Efficiency Metric.}  
Define length-normalized accuracy
\begin{equation}
\mathrm{L\text{-}Acc}(L) = \mathrm{Acc}(L)\,\sqrt{1 - \frac{L}{L_{\max}}}.
\end{equation}
In practice, dynamic scheduling often achieves similar or higher accuracy with shorter or comparable length, leading to
\begin{equation}
\mathrm{L\text{-}Acc}(L_d) > \mathrm{L\text{-}Acc}(L_s).
\end{equation}

\bigskip
\noindent\textbf{Conclusion.}  
Our dynamic length reward realizes the curriculum
\begin{equation}
\begin{split}
\text{explore freely }(\lambda \leq 0)
\;\longrightarrow\;
\text{accuracy convergence} \\
\;\longrightarrow\;
\text{gradual compression }(\lambda\uparrow).
\end{split}
\end{equation}
This schedule lets the model reach its accuracy ceiling $\mathrm{Acc}(L^\star)$ before enforcing brevity, achieving better accuracy–efficiency trade-offs than static schemes.

\section{Definition and Theoretical Analysis of Length-normalized Accuracy}
\label{ap:l-acc}

\noindent
\textbf{Length-Normalized Accuracy.}

To evaluate reasoning efficiency, we adopt a length-normalized accuracy metric, denoted as \textsc{L-Acc}, which balances correctness with brevity. Formally, it is defined as:
\begin{equation}
\mathrm{L\text{-}Acc} = \mathrm{Acc} \times \sqrt{1 - \frac{L}{L_{\max}}}, \label{eq:len_acc}
\end{equation}
where $\mathrm{Acc} \in [0,1]$ denotes exact-match accuracy, $L$ is the number of tokens in the model’s response, and $L_{\max}$ is a dataset-specific upper bound on response length. The multiplicative factor penalizes longer outputs in a sub-linear manner, rewarding models that solve problems with fewer tokens.

\begin{figure}[t]
  \centering
  \includegraphics[width=0.48\textwidth]{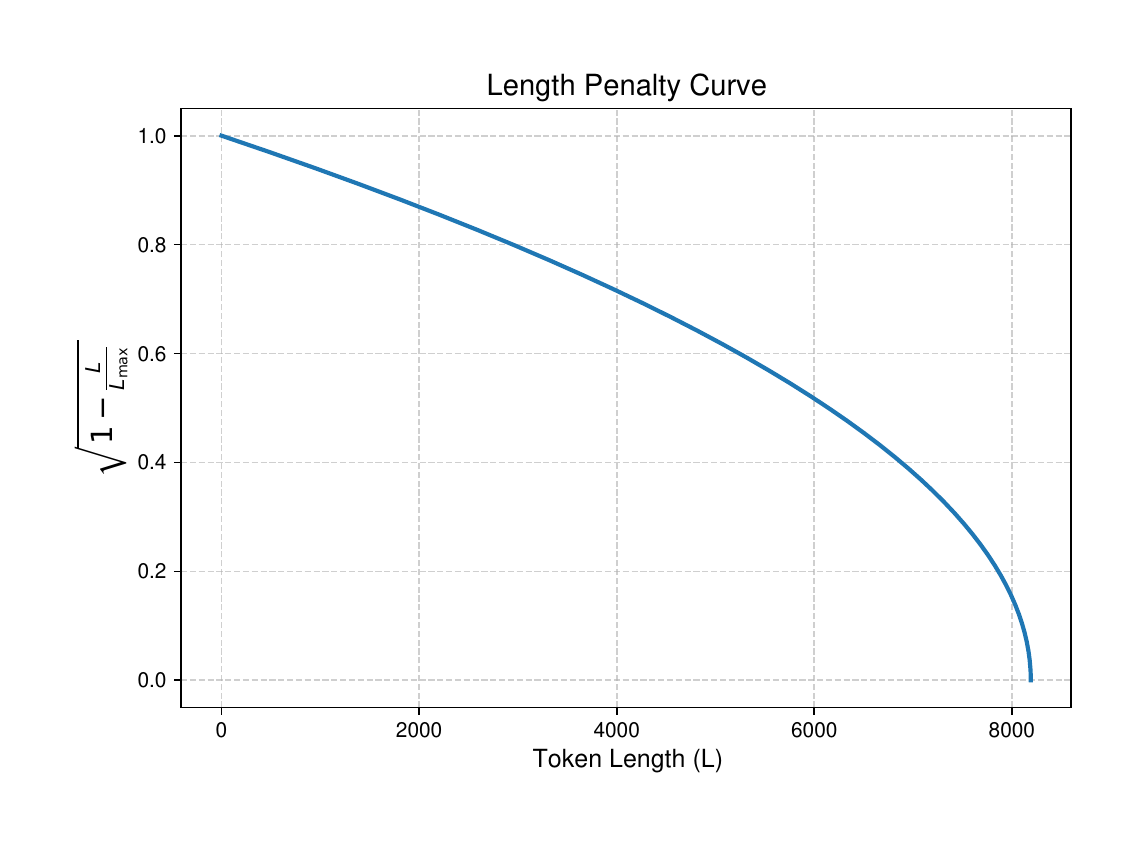}
  \caption{\textbf{Penalty curve:} $\sqrt{1 - \frac{L}{L_{\max}}}$.}
  \label{fig:lacc}
\end{figure}

Specifically, we set \(L_{\max}=8192\) for the two \textit{DeepSeek}-based reasoning models, and \(L_{\max}=3000\) for the \textit{Qwen2.5-Math-7B-Instruct} model, since reasoning-oriented models generally generate longer outputs than instruction-tuned models. The multiplicative factor \(\sqrt{1-\frac{L}{L_{\max}}}\) weights accuracy by a sub-linear penalty on sequence length, so the metric rewards correct solutions that
are delivered with fewer tokens. Normalizing by \(L_{\max}\) makes the score
comparable across datasets of very different scale, while the square-root ensures a smooth, continuous trade-off: the first tokens cut away improve the score more than later ones, mirroring human tolerance for moderate verbosity but aversion to extreme length. When \(L=L_{\max}\) the metric collapses to zero, preventing models from exchanging unbounded length for marginal accuracy gains; when \(L=0\) it reduces to the raw accuracy, preserving credit for perfectly concise answers.

\noindent
\textbf{Penalty Behavior and Physical Intuition.}
The penalty term $\sqrt{1 - \frac{L}{L_{\max}}}$ is continuous, monotonically decreasing, and bounded between 0 and 1. It applies no penalty when $L = 0$, and reduces the reward to zero when $L = L_{\max}$, even if the answer is correct. Crucially, the square-root form introduces diminishing returns: trimming early redundant tokens provides larger gains in \textsc{L-Acc} than removing tokens later in the sequence. This design mirrors human preferences—we tolerate moderate verbosity, but disfavor excessive detail. It also echoes the behavior of L2 regularization, where larger values are penalized more aggressively, while smaller deviations are softly constrained.

\noindent
\textbf{Gradient Analysis.}
To understand its optimization implications, we analyze the gradient of the penalty term with respect to $L$:
\begin{equation}
\frac{d}{dL} \left( \sqrt{1 - \frac{L}{L_{\max}}} \right) = -\frac{1}{2L_{\max}} \cdot \left(1 - \frac{L}{L_{\max}}\right)^{-1/2}.
\end{equation}
This derivative diverges as $L \to L_{\max}$, indicating that long outputs are heavily penalized. In contrast, when $L$ is small, the gradient approaches zero, and the penalty becomes negligible. This behavior encourages models to first eliminate highly redundant tokens, while maintaining stability for shorter outputs.

\noindent
\textbf{Optimization Benefits.}
Unlike hard constraints on length, this formulation yields a smooth and differentiable reward signal, making it well-suited for reinforcement learning algorithms such as PPO. It provides stable guidance throughout training and enables the model to trade off between accuracy and length in a controlled and interpretable manner. As shown in Figure~\ref{fig:lacc}, the penalty curve strongly discourages excessively long outputs while allowing flexibility in moderately verbose cases, contributing to more efficient and human-aligned reasoning. Notably, the curve becomes steep as the response length approaches \(L_{\max}\), meaning that small increases in length lead to sharp drops in reward; conversely, it flattens near \(L = 0\), where changes in length have only a minor effect on the reward. This property ensures that the model is heavily penalized for extreme verbosity while remaining tolerant of brief explanatory content.

\section{Detailed Settings of Experiments}
\label{ap:detail_setting}

\noindent\textbf{Prompt.} 
All experiments use the prompt: \texttt{"Let's think step by step and output the final answer within \textbackslash boxed\{\}."}

\noindent\textbf{Models.} 
Our experiments involve a mix of proprietary and open-source models. The models evaluated in this study include:

\begin{itemize}[leftmargin=20pt, itemsep=0pt, labelsep=5pt, topsep=0pt]
    \item \textbf{DeepSeek-R1-Distill-Qwen-1.5B (MIT License)}: A fine-tuned model with 1.5 billion parameters, used to evaluate the proposed method.
    \item \textbf{DeepSeek-R1-Distill-Qwen-7B (MIT License)}: A fine-tuned model with 7 billion parameters, also used to evaluate the proposed method.
    \item \textbf{Qwen2.5-Math-7B-Instruct (Apache-2.0 License)}: An instruction-tuned model with 7 billion parameters, used to further assess the efficiency and accuracy in reasoning tasks.
\end{itemize}

\noindent\textbf{Datasets.} 
We evaluate our models on several datasets covering both in-distribution (ID) and out-of-distribution (OOD) tasks. The evaluation framework encompasses:

\begin{itemize}[leftmargin=20pt, itemsep=0pt, labelsep=5pt, topsep=0pt]
  \item \textbf{MATH}~\citep{hendrycksmath2021}: A comprehensive training dataset containing 7,500 mathematical problems across various difficulty levels and topics.
  \item \textbf{MATH500}~\citep{hendrycksmath2021}: A carefully selected 500-problem subset from the MATH test set, serving as our primary in-distribution evaluation benchmark.
  \item \textbf{GSM8K}~\citep{cobbe2021training}: Grade school math word problems requiring multi-step reasoning, used for out-of-distribution evaluation on elementary-level mathematics.
  \item \textbf{TheoremQA}~\citep{chen-etal-2023-theoremqa}: A challenging dataset requiring theorem application and symbolic reasoning across STEM domains, used for out-of-distribution evaluation.
  \item \textbf{AIME2024}~\citep{aime_1983_2024}: Problems from the prestigious American Invitational Mathematics Examination, representing the most challenging out-of-distribution evaluation.
\end{itemize}

These datasets are arranged in increasing order of difficulty: \textsc{GSM8K} $<$ \textsc{MATH500} $<$ \textsc{TheoremQA} $<$ \textsc{AIME2024}, offering a comprehensive evaluation of models' reasoning capabilities across varying complexity levels, as summarized in Table~\ref{tab:datasets}.

\begin{table*}[ht]
\centering
\caption{Overview of datasets used for training and evaluation.}
\resizebox{\textwidth}{!}{%
\begin{tabular}{llrrcccc}
\toprule
\textbf{Type} & \textbf{Dataset} & \textbf{\# Train} & \textbf{\# Test} & \textbf{Domain} & \textbf{Task Type} & \textbf{Difficulty} & \textbf{Source} \\
\midrule
Training & MATH~\citep{hendrycksmath2021} & 7,500 & -- & Mathematics & Problem Solving & Mixed & \href{https://huggingface.co/datasets/DigitalLearningGmbH/MATH-lighteval}{Link} \\
\midrule
ID Test & MATH500~\citep{hendrycksmath2021} & -- & 500 & Mathematics & Problem Solving & Medium-Hard & \href{https://huggingface.co/datasets/HuggingFaceH4/MATH-500}{Link} \\
\midrule
\multirow{3}{*}{OOD Test} 
  & GSM8K~\citep{cobbe2021training} & -- & 1,319 & Elementary Math & Word Problems & Easy & \href{https://huggingface.co/datasets/openai/gsm8k}{Link} \\
  & TheoremQA~\citep{chen-etal-2023-theoremqa} & -- & 800 & STEM & Theorem Application & Hard & \href{https://huggingface.co/datasets/TIGER-Lab/TheoremQA}{Link} \\
  & AIME2024~\citep{aime_1983_2024} & -- & 30 & Competition Math & Advanced Problem Solving & Very Hard & \href{https://huggingface.co/datasets/sea-snell/aime-2024}{Link} \\
\bottomrule
\end{tabular}
} 
\label{tab:datasets}
\end{table*}

\noindent\textbf{Preprocessing and Tokenization.} 
Each model uses its corresponding tokenizer to process the input sequences. Tokenization ensures compatibility with the model's input structure, using special tokens to denote the start and end of sequences. 



\noindent\textbf{Training Procedure.} 
All models are trained for a total of 50 epochs using the Proximal Policy Optimization (PPO) algorithm, optimizing for both accuracy and efficiency. The actor and critic models are initialized with the same parameters, and training is conducted with the following hyperparameters: actor learning rate $=5\times10^{-5}$, critic learning rate $=1\times10^{-6}$, mini-batch size $=512$, and KL-divergence coefficient $=0.001$. Evaluation is performed periodically at every training step to monitor progress, and the best model checkpoints are selected for final testing. 

\noindent\textbf{Decoding Configurations.}
We conduct both training and evaluation under carefully controlled decoding settings. During training, we adopt sampling generation with temperature $=1.0$, $\text{top\_k}=-1$, and $\text{top\_p}=1.0$ to encourage exploration, following the default configuration of the \textsc{VERL} framework for comparability with prior work. A single response ($n=1$) is generated per prompt, with the maximum prompt length capped at $1{,}024$ tokens for efficiency. The maximum response length is set to $8{,}192$ tokens for DeepSeek models and $3{,}000$ tokens for Qwen-Math models.  

For evaluation, we emphasize efficiency and stability by adopting greedy decoding, consistent with \textsc{VERL} defaults and prior studies \citep{yeo2025demystifyinglongchainofthoughtreasoning, cui2025entropy}. Specifically, evaluation uses greedy decoding with temperature $=0$, and one response per input ($n=1$). The same maximum response lengths as in training are applied ($8{,}192$ for DeepSeek, $3{,}000$ for Qwen-Math). We also conducted experiments under extended sampling configurations, with comprehensive results presented in Appendix~\ref{ap:sampling}.

\noindent\textbf{Evaluation Metrics.} 
We evaluate the models using the following metrics:

\begin{itemize}[leftmargin=20pt, itemsep=0pt, labelsep=5pt, topsep=0pt]
  \item \textbf{Exact Match (EM)}: Measures the proportion of exact matches between the generated output and the ground-truth answer.
  \item \textbf{Response Length (Len)}: Measures the number of tokens in the output sequence.
  \item \textbf{Length-Normalized Accuracy (L-Acc)}: A metric that balances accuracy and efficiency by considering both correctness and response length.
\end{itemize}

\noindent\textbf{Baselines.} 
We compare the \textsc{Bingo} framework with the following baselines:
\begin{itemize}[leftmargin=20pt, itemsep=0pt, labelsep=5pt, topsep=0pt]
  \item \textbf{DAST}~\citep{shi2025dastdifficultyadaptiveslowthinking}: Uses dynamic length penalties based on problem difficulty.
  \item \textbf{Efficient Reasoning}~\citep{arora2025traininglanguagemodels}: Scales down positive rewards to encourage brevity.
  \item \textbf{Kimi-k1.5}~\citep{kimiteam2025kimik15scalingreinforcement}: Applies online length penalties.
  \item \textbf{O1-Pruner}~\citep{luo2025o1prunerlengthharmonizingfinetuningo1like}: Applies offline length penalties based on length comparisons with reference sequences.
  \item \textbf{Demystifying}~\citep{yeo2025demystifyinglongchainofthoughtreasoning}: Applies a symmetric penalty strategy for response lengths, encouraging both shorter and more extensive reasoning depending on correctness.
\end{itemize}

Rather than directly comparing published baselines—which employ diverse frameworks (e.g., SimPO, GRPO) and differ in their on-policy versus off-policy implementations—we isolate and re-implement only the length-based reward components proposed in each work. All these reward designs are integrated into a unified PPO framework. This approach enables a fair comparison focused specifically on the effectiveness of different reward formulations for improving reasoning efficiency.

\noindent\textbf{Software and Hardware.} 
The experiments are conducted with Python 3.11, PyTorch v2.4.0, and CUDA 12.8 for model training and inference. We use 4 NVIDIA A100 80GB PCIe GPUs for training the 7B model and 2 NVIDIA H100 80GB PCIe GPUs for training the 1.5B model. For inference, 2 NVIDIA H100 80GB PCIe GPUs are used to accelerate processing.

\begin{table*}[ht]
\centering
\caption{\textbf{Performance comparison under sampling decoding settings.} 
Each method is evaluated using DeepSeek-R1-Distill-Qwen-1.5B as the base model with sampling parameters (32,768 token limit, 3 samples, temperature = 0.6, top-p = 1.0). Metrics include answer accuracy (Acc, \%), response length (Len), and length-normalized accuracy (L-Acc, \%). The best performance is highlighted in \textcolor{langdarkblue}{dark blue}, and the second-best in \textcolor{langwildblue}{light blue}.}
\resizebox{\textwidth}{!}{%
\begin{tabular}{llcccccccccccc}
\toprule
\multicolumn{2}{l}{\textbf{Method}} 
& \multicolumn{3}{c}{\textbf{MATH500}} 
& \multicolumn{3}{c}{\textbf{GSM8K}} 
& \multicolumn{3}{c}{\textbf{TheoremQA}} 
& \multicolumn{3}{c}{\textbf{AIME2024}} \\
\cmidrule(lr){3-5} \cmidrule(lr){6-8} \cmidrule(lr){9-11} \cmidrule(lr){12-14}
& & Acc$\uparrow$ & Len$\downarrow$ & L-Acc$\uparrow$ 
    & Acc$\uparrow$ & Len$\downarrow$ & L-Acc$\uparrow$ 
    & Acc$\uparrow$ & Len$\downarrow$ & L-Acc$\uparrow$ 
    & Acc$\uparrow$ & Len$\downarrow$ & L-Acc$\uparrow$ \\
\midrule
& Base             & 81.6 & 5,155 & 74.9 & 83.7 & 1,748 & 81.4 & 31.7 & 7,598 & 27.8 & 17.8 & 15,703 & 12.8 \\
& Vanilla PPO      & 82.3 & 2,694 & 78.8 & 86.5 & 1,050 & 85.1 & 33.4 & 3,616 & 31.5 & 28.9 & 7,389 & 25.4 \\
& O1-Pruner        & 80.1 & 1,283 & 78.5 & 85.2 & \textcolor{langwildblue}{\textbf{352}} & 84.7 & 34.7 & \textcolor{langwildblue}{\textbf{1,095}} & 34.1 & 28.9 & 4,636 & 26.8 \\
& Demystifying     & 81.3 & 1,945 & 78.9 & 86.8 & 483 & 86.2 & 35.2 & 1,863 & 34.2 & 30.0 & 5,891 & 27.2 \\
& DAST             & 83.5 & 2,053 & 80.8 & 84.1 & 375 & 83.6 & 35.4 & 2,954 & 33.8 & 36.7 & 5,072 & 33.7 \\
\midrule
\multicolumn{2}{l}{\textbf{\textit{Bingo (Ours)}}} & & & & & & & & & & & & \\
& \textbf{Bingo-A} 
    & \textcolor{langdarkblue}{\textbf{85.1}} 
    & \textcolor{langwildblue}{\textbf{1,114}} 
    & \textcolor{langdarkblue}{\textbf{83.6}} 
    & \textcolor{langdarkblue}{\textbf{88.4}} 
    & 483 
    & \textcolor{langwildblue}{\textbf{87.7}} 
    & \textcolor{langdarkblue}{\textbf{37.9}} 
    & 1,592 
    & \textcolor{langwildblue}{\textbf{37.0}} 
    & \textcolor{langdarkblue}{\textbf{38.9}} 
    & \textcolor{langdarkblue}{\textbf{3,110}} 
    & \textcolor{langdarkblue}{\textbf{37.0}} \\
& \textbf{Bingo-E} 
    & \textcolor{langwildblue}{\textbf{84.2}} 
    & \textcolor{langdarkblue}{\textbf{983}} 
    & \textcolor{langwildblue}{\textbf{82.9}} 
    & \textcolor{langwildblue}{\textbf{88.1}} 
    & \textcolor{langdarkblue}{\textbf{217}} 
    & \textcolor{langdarkblue}{\textbf{87.8}} 
    & \textcolor{langwildblue}{\textbf{37.7}} 
    & \textcolor{langdarkblue}{\textbf{1,004}} 
    & \textcolor{langdarkblue}{\textbf{37.1}} 
    & \textcolor{langwildblue}{\textbf{37.6}} 
    & \textcolor{langwildblue}{\textbf{2,817}} 
    & \textcolor{langwildblue}{\textbf{35.9}} \\
\bottomrule
\end{tabular}%
}
\label{tab:sampling}
\vspace{-0.7em}
\end{table*}

\begin{figure*}[h]
  \centering
  \begin{subfigure}[t]{0.43\textwidth}
    \centering
    \includegraphics[height=120pt]{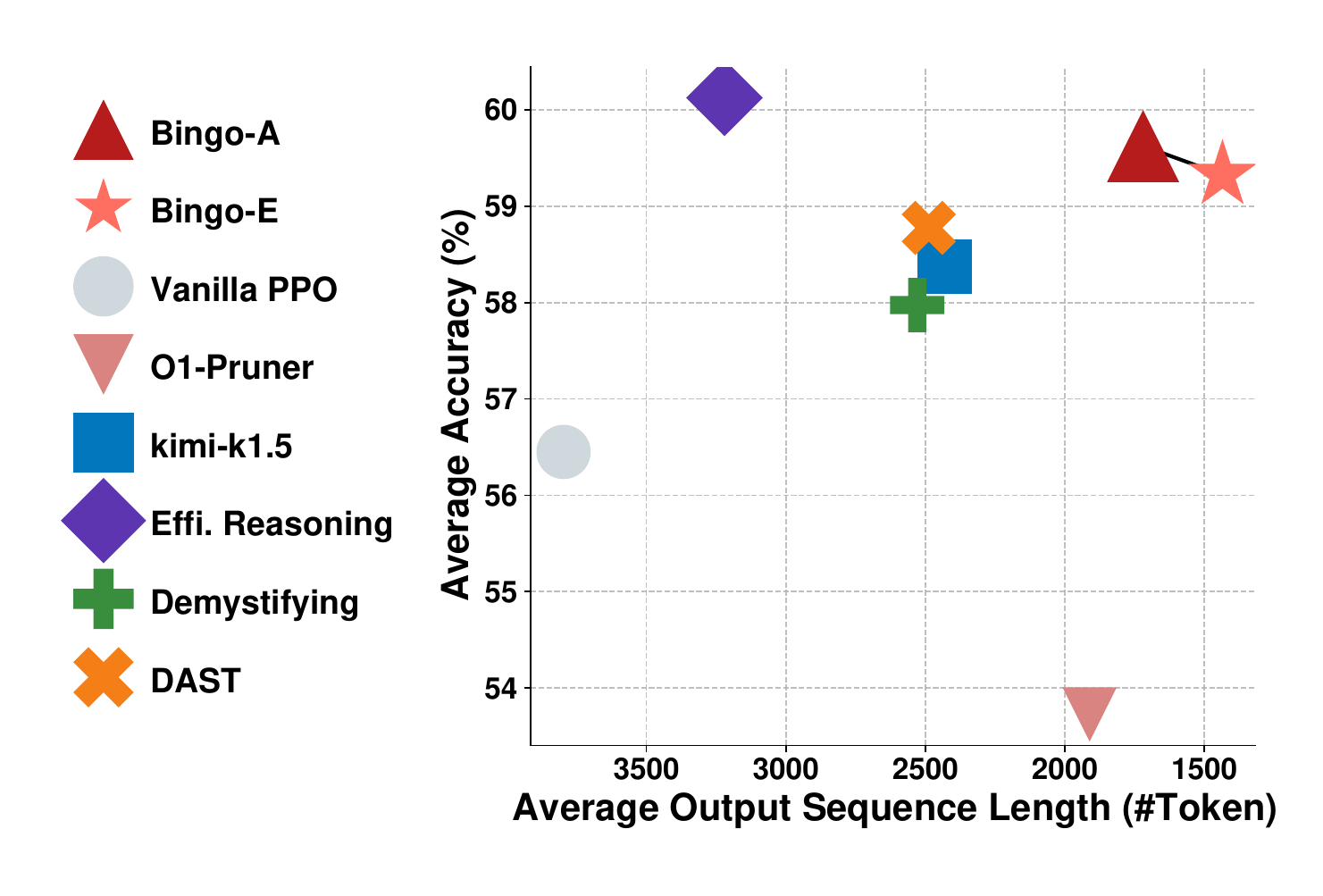}
    \label{fig:teaser_scatter}
  \end{subfigure}
  \begin{subfigure}[t]{0.5\textwidth}
    \centering
    \includegraphics[height=120pt]{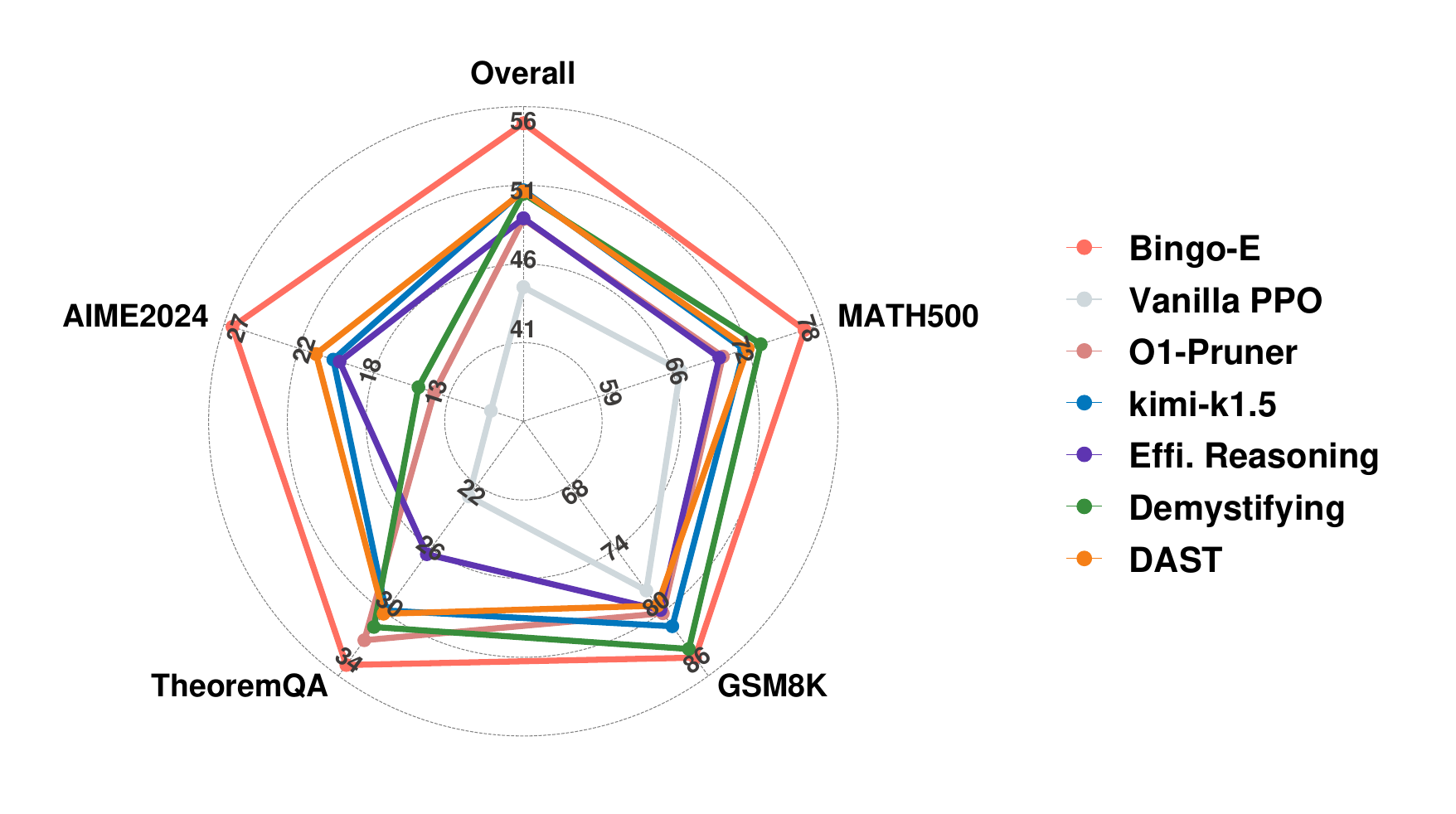}
    \label{fig:teaser_radar}
  \end{subfigure}
  \vspace{-1.2em}
  \caption{\textbf{Performance overview of \textsc{Bingo} and other baselines.} \textbf{\textit{Left}}: Scatter plot of average accuracy versus average response length on four benchmarks (MATH500, GSM8K, TheoremQA, AIME2024) using DeepSeek-R1-Distill-Qwen-1.5B as the base model. Points nearer the top‑right corner represent a better balance of accuracy and efficiency. \textbf{\textit{Right}}: Radar chart of length‑normalized accuracy for each method. Greater radial distances denote higher efficiency.}
  \label{fig:teaser_baseline}
\end{figure*}

\section{Performance under Extended Sampling Settings}
\label{ap:sampling}

We conducted additional experiments using sampling decoding to assess the robustness of our approach under more exploratory conditions. These experiments employed an extended configuration with a 32,768-token output limit, three samples per prompt, temperature of 0.6, and top-p of 1.0. We evaluated the base DeepSeek-R1-Distill-Qwen-1.5B model, vanilla PPO, our proposed Bingo method, and selected competitive baselines to ensure comprehensive comparison.

Table~\ref{tab:sampling} presents the accuracy and average response length across four benchmarks under these sampling conditions. The results demonstrate that Bingo maintains its efficiency advantage even with sampling decoding, achieving strong accuracy while generating substantially shorter outputs than all baseline methods. This finding confirms that our reward design effectively promotes concise reasoning regardless of the decoding strategy employed.

Although these extended settings yielded accuracy improvements, they required approximately five times the computational resources and training time compared to greedy decoding. Given this substantial computational overhead, we selected single-response greedy decoding as our primary evaluation protocol to maintain experimental feasibility while still providing meaningful performance assessments. The sampling results presented here validate that our approach remains effective under more computationally intensive conditions.

\section{Performance overview of \textsc{Bingo} and other baselines}
\label{ap:baseline}

Figure~\ref{fig:teaser_baseline} provides an overview of the performance of \textsc{Bingo} compared with strong baselines in terms of both accuracy and efficiency. The scatter plot shows that \textsc{Bingo} variants consistently occupy the favorable region with higher accuracy and substantially shorter response lengths, indicating a superior accuracy--efficiency trade-off. This advantage holds across both relatively simple benchmarks such as GSM8K and more challenging settings like TheoremQA, where \textsc{Bingo} achieves noticeable accuracy improvements while reducing response length by more than half. The radar chart further confirms these trends, with \textsc{Bingo} attaining the highest length-normalized accuracy across datasets, demonstrating its ability to generate concise yet effective reasoning across diverse difficulty levels.

\section{Performance across Different Reinforcement Learning Algorithms}
\label{ap:rl}
\begin{table*}[t]
\centering
\scriptsize
\setlength{\tabcolsep}{4pt}
\caption{\textbf{Comparison of reinforcement learning algorithms on four reasoning benchmarks.} 
Each method is evaluated using DeepSeek-R1-Distill-Qwen-1.5B as the base model by answer accuracy (Acc, \%), response length (Len), and length-normalized accuracy (L-Acc, \%). Bingo-based variants consistently outperform their vanilla counterparts across different RL optimizers (PPO, RLOO, GRPO, Reinforce++). Numbers in parentheses show the L-Acc gain over the corresponding vanilla baseline, with \textcolor{langgreen}{\textbf{green}} indicating improvement.}
\resizebox{\textwidth}{!}{%
\begin{tabular}{lccccccccccccc}
\toprule
\multicolumn{1}{l}{\textbf{Method}}
  & \multicolumn{3}{c}{\textbf{MATH500}}
  & \multicolumn{3}{c}{\textbf{GSM8K}}
  & \multicolumn{3}{c}{\textbf{TheoremQA}}
  & \multicolumn{3}{c}{\textbf{AIME2024}} \\
\cmidrule(lr){2-4} \cmidrule(lr){5-7} \cmidrule(lr){8-10} \cmidrule(lr){11-13}
  & Acc$\uparrow$ & Len$\downarrow$ & L-Acc$\uparrow$
  & Acc$\uparrow$ & Len$\downarrow$ & L-Acc$\uparrow$
  & Acc$\uparrow$ & Len$\downarrow$ & L-Acc$\uparrow$
  & Acc$\uparrow$ & Len$\downarrow$ & L-Acc$\uparrow$ \\
\midrule
Base
  & 63.2 & 3,913 & 45.7  
  & 73.2 & 2,025 & 63.5  
  & 18.7 & 5,741 & 10.3  
  & 16.7 & 7,027 & 6.3  \\
\midrule
Vanilla PPO      
  & 81.4 & 2,771 & 66.2  
  & 85.4 & 1,310 & 78.2  
  & 32.3 & 4,146 & 22.7  
  & 26.7 & 6,961 & 10.3 \\
\textbf{Bingo-PPO}
  & 82.2 &   894 & 77.6 \textcolor{langgreen}{\textbf{(+11.4)}}
  & 87.0 &   563 & 83.9 \textcolor{langgreen}{\textbf{(+5.7)}}
  & 36.8 & 1,648 & 32.9 \textcolor{langgreen}{\textbf{(+10.2)}}
  & 33.3 & 2,943 & 26.7 \textcolor{langgreen}{\textbf{(+16.4)}} \\
\midrule
Vanilla RLOO
  & 76.8 & 2,413 & 64.5  
  & 77.3 & 1,588 & 69.4  
  & 30.0 & 3,162 & 23.5  
  & 26.7 & 6,025 & 13.7 \\
\textbf{Bingo-RLOO}
  & 78.0 & 1,985 & 67.9 \textcolor{langgreen}{\textbf{(+3.4)}}
  & 80.7 &   450 & 78.5 \textcolor{langgreen}{\textbf{(+9.1)}}
  & 32.0 & 2,230 & 27.3 \textcolor{langgreen}{\textbf{(+3.8)}}
  & 33.3 & 5,583 & 18.8 \textcolor{langgreen}{\textbf{(+5.1)}} \\
\midrule
Vanilla GRPO
  & 76.4 & 2,533 & 63.5  
  & 77.8 &   804 & 73.9  
  & 29.2 & 2,946 & 23.4  
  & 26.7 & 6,096 & 13.5 \\
\textbf{Bingo-GRPO}
  & 79.4 & 1,753 & 70.4 \textcolor{langgreen}{\textbf{(+6.9)}}
  & 80.0 &   449 & 77.8 \textcolor{langgreen}{\textbf{(+3.9)}}
  & 31.9 & 2,298 & 27.0 \textcolor{langgreen}{\textbf{(+3.6)}}
  & 30.0 & 5,886 & 15.9 \textcolor{langgreen}{\textbf{(+2.4)}} \\
\midrule
Vanilla Reinforce++
  & 76.2 & 2,842 & 61.6  
  & 82.0 & 1,291 & 75.2  
  & 28.0 & 3,977 & 20.1  
  & 30.0 & 6,168 & 14.9 \\
\textbf{Bingo-Reinforce++}
  & 78.4 & 2,070 & 67.8 \textcolor{langgreen}{\textbf{(+6.2)}}
  & 81.0 &   640 & 77.8 \textcolor{langgreen}{\textbf{(+2.6)}}
  & 33.1 & 2,566 & 27.4 \textcolor{langgreen}{\textbf{(+7.3)}}
  & 30.0 & 5,885 & 15.9 \textcolor{langgreen}{\textbf{(+1.0)}} \\
\bottomrule
\end{tabular}%
}
\label{tab:rl}
\end{table*}

To evaluate the generalizability of our reward design, we integrate it into multiple on-policy reinforcement learning (RL) algorithms, including PPO, RLOO, GRPO, and Reinforce++. As shown in Table~\ref{tab:rl}, Bingo-enhanced variants consistently outperform their vanilla counterparts across all four benchmarks in both accuracy and length-normalized accuracy (L-Acc). Among them, \textsc{Bingo-PPO} delivers the strongest overall performance, achieving the highest or second-highest scores on all datasets while substantially reducing output length. Moreover, the benefits of Bingo extend beyond PPO: each Bingo variant improves L-Acc over its baseline by a clear margin, demonstrating that our reward formulation generalizes well across different policy optimization strategies. These gains are observed not only on in-distribution datasets like \textsc{MATH500} and \textsc{GSM8K}, but also on more challenging out-of-distribution settings such as \textsc{AIME2024}, highlighting the robustness of our approach. Overall, the results confirm that a principled and learnable length-aware reward offers a consistent advantage across a variety of reasoning tasks and RL algorithms.

\begin{figure*}[t!]
    \centering
    \includegraphics[width=0.95\textwidth]{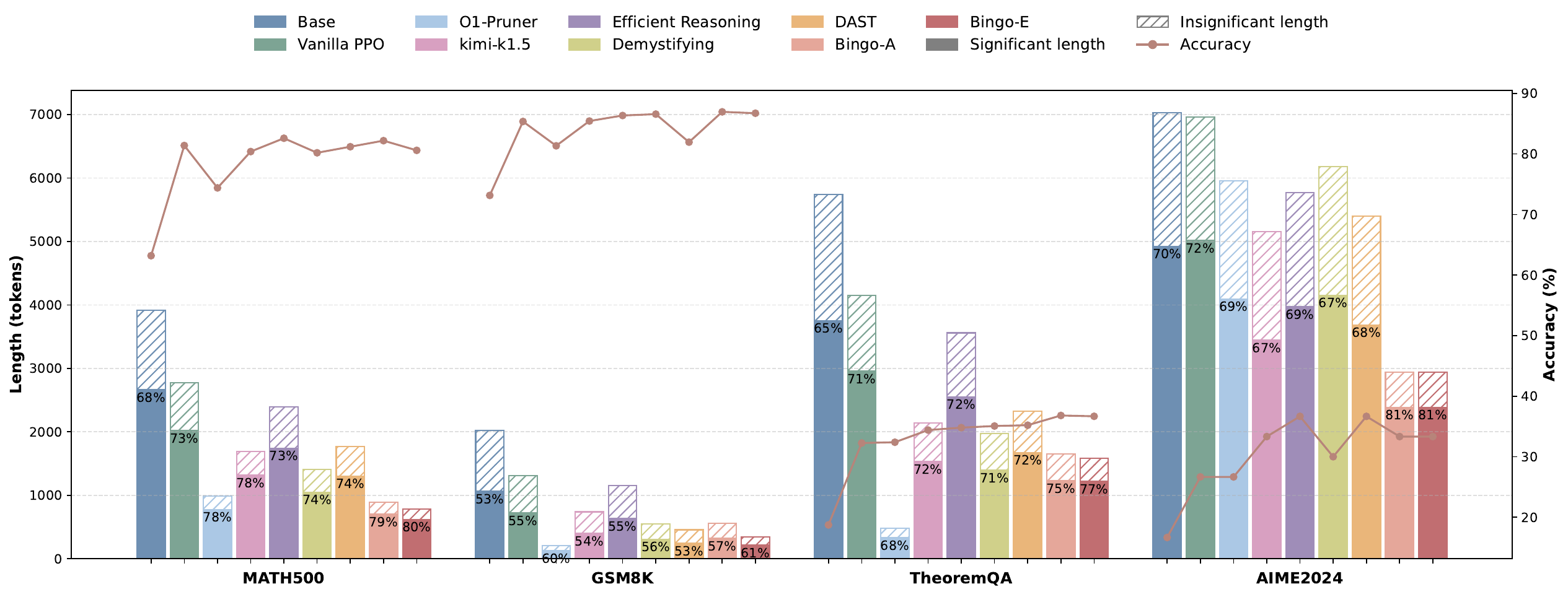}
    \caption{\textbf{Length–accuracy results for nine optimization algorithms on four datasets.} Bars show the number of tokens generated using DeepSeek-R1-Distill-Qwen-1.5B as the base model at the checkpoint that yields the reported accuracy (left axis). Each bar is partitioned into significant (dark) and insignificant (light) segments, and the percentage above the bar indicates the share of significant tokens. The solid line (right axis) gives the corresponding answer accuracy. Our methods, Bingo-A and Bingo-E, attain the highest accuracy while using the fewest tokens and achieving the greatest proportion of significant tokens, highlighting their superior reasoning efficiency.}
    \label{fig:len_acc_ratio}
\end{figure*}

\section{Analysis of Significant Token Ratio}
\label{ap:sig_token_ratio}

We trained \texttt{DeepSeek-R1-Distill-Qwen-1.5B} exclusively on the \textsc{MATH} corpus. To evaluate its generalization beyond the training distribution, we test the model on the in-distribution split \textsc{MATH500} as well as three out-of-distribution (OOD) benchmarks: \textsc{GSM8K}, \textsc{TheoremQA}, and \textsc{AIME2024}. Figure~\ref{fig:len_acc_ratio} shows that our approaches, \textbf{Bingo-A} and \textbf{Bingo-E}, achieve the most favorable length-accuracy trade-off across all four benchmarks.
\begin{itemize}[leftmargin=*,itemsep=0pt,topsep=2pt]
    \item \textbf{Efficiency at peak accuracy.} At the checkpoints that obtain their highest accuracy, both Bingo variants require only about \(\!20\%\) of the tokens used by the \emph{Base} model on \textsc{MATH500}, with similarly large reductions on \textsc{GSM8K}, \textsc{TheoremQA}, and \textsc{AIME2024}.
    \item \textbf{Preservation of informative content.} Bingo increases the share of significant tokens to \(75\text{--}81\%\), showing that the shortened rationales shed mainly redundant rather than essential reasoning steps.
    \item \textbf{Difficulty-dependent length trends.} Token counts grow with task difficulty: the two harder benchmarks, \textsc{TheoremQA} and \textsc{AIME2024}, demand considerably longer rationales and yield lower absolute accuracy than \textsc{MATH500} and \textsc{GSM8K}. Even under these tougher conditions, Bingo still delivers the highest accuracy while generating the fewest tokens.

    \item \textbf{Alleviating the length--accuracy trade-off.} Baselines that compress reasoning without accounting for token importance (e.g., \emph{O1-Pruner}) exhibit marked accuracy declines, whereas Bingo maintains---and in some cases slightly improves---task performance.
    \item \textbf{Robustness across tasks.} The same advantage holds for algebraic, commonsense, formal-logic, and competition-style benchmarks, underscoring the generality of the significance-aware and dynamic length rewards.
\end{itemize}
These findings confirm that explicitly modeling token significance and adaptively scheduling length rewards enables language models to reason both \emph{accurately} and \emph{efficiently}.

\begin{figure*}[t]
    \centering
    \resizebox{\textwidth}{!}{
    \fbox{\parbox[c]{1.1\textwidth}{
        \textbf{Problem: } Terry eats 2 yogurts a day. They are currently on sale at 4 yogurts for \$5.00. How much does he spend on yogurt over 30 days?
        \vskip 0.02in
        \textbf{Response of \textsc{Bingo}: } {\setlength{\fboxsep}{-1pt}
         \input{figures/bingo_token}
        }
        \vskip 0.02in
        \textbf{Response of Vanilla PPO: } {\setlength{\fboxsep}{-1pt}
         \input{figures/base_token}
        }
        \vskip 0.02in
        \textbf{Final Answer: } 75.
    }}}
    \caption{\textbf{Token-level significance visualization for a sample reasoning task.} 
Each token is colored based on its predicted significance: red indicates significant tokens (darker = more significant), and blue indicates insignificant tokens (darker = less significant). 
The response from \textsc{Bingo} (top) is shorter and more concentrated around meaningful reasoning steps, while the Vanilla PPO response (bottom) is longer and contains more exploratory and redundant language. 
The visualization illustrates how Bingo encourages more efficient and focused reasoning.}
    \label{fig:token-importance}
\end{figure*}

\section{Token-level Significance Visualization}
\label{ap:sig_visualization}

Figure \ref{fig:token-importance} provides a token-level significance visualization for a sample reasoning task. The problem involves calculating the cost of yogurt based on a given sale, and both the \textsc{Bingo} and Vanilla PPO models generate responses to solve the problem. Each token in the generated response is color-coded based on its predicted significance, with red indicating significant tokens and blue representing insignificant ones. Darker shades of red and blue correspond to higher significance levels.

The response from \textsc{Bingo} (top) is notably shorter and more focused on the key reasoning steps, highlighting the model's ability to concentrate on relevant tokens while avoiding unnecessary elaboration. In contrast, the Vanilla PPO response (bottom) is longer, with a higher proportion of redundant and less informative tokens, reflecting a less efficient reasoning process. This visualization clearly demonstrates how \textsc{Bingo} encourages more concise and targeted reasoning, optimizing for both accuracy and efficiency by emphasizing significant steps in the reasoning process.



\begin{figure*}[t]
    \centering
    \includegraphics[width=0.95\textwidth]{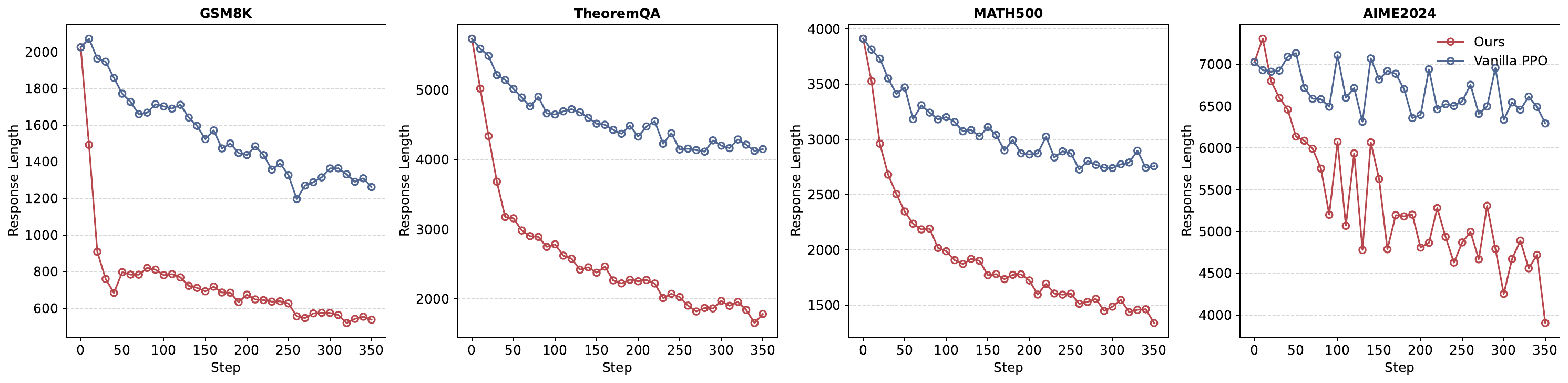}
    \caption{\textbf{Response length trends during training across four datasets.} The y-axis shows the number of tokens generated per response using DeepSeek-R1-Distill-Qwen-1.5B as the base model; the x-axis denotes training steps. The red line represents our method, and the blue line corresponds to Vanilla PPO.
Across all tasks, our method consistently produces shorter and more stable responses, demonstrating improved reasoning efficiency without compromising task performance.}
    \label{fig:length}    
\end{figure*}

\section{Analysis of Response Lengths Trends during Training}
\label{ap:response_trends}

Figure~\ref{fig:length} presents the evolution of response length over training steps for Vanilla PPO and our method on four benchmarks. Our approach consistently yields substantially shorter outputs than Vanilla PPO throughout training, demonstrating effective removal of redundant tokens, and converges more smoothly, reflecting robust length regularization. The reduction in response length is most pronounced on the more demanding tasks—\textsc{MATH500} and \textsc{AIME2024}—where Vanilla PPO produces very long sequences, yet our method maintains a compact reasoning footprint. Importantly, this improvement generalizes across diverse reasoning styles, from arithmetic problems in \textsc{GSM8K} and formal-logic questions in \textsc{TheoremQA} to academic and competition-style challenges, confirming that our reward design enhances reasoning efficiency without compromising training stability.

\section{Analysis of Response Lengths Dynamics for Correct vs. Wrong Samples}
\label{ap:response_c_w}

\begin{figure*}[t]
    \centering
    \includegraphics[width=0.95\textwidth]{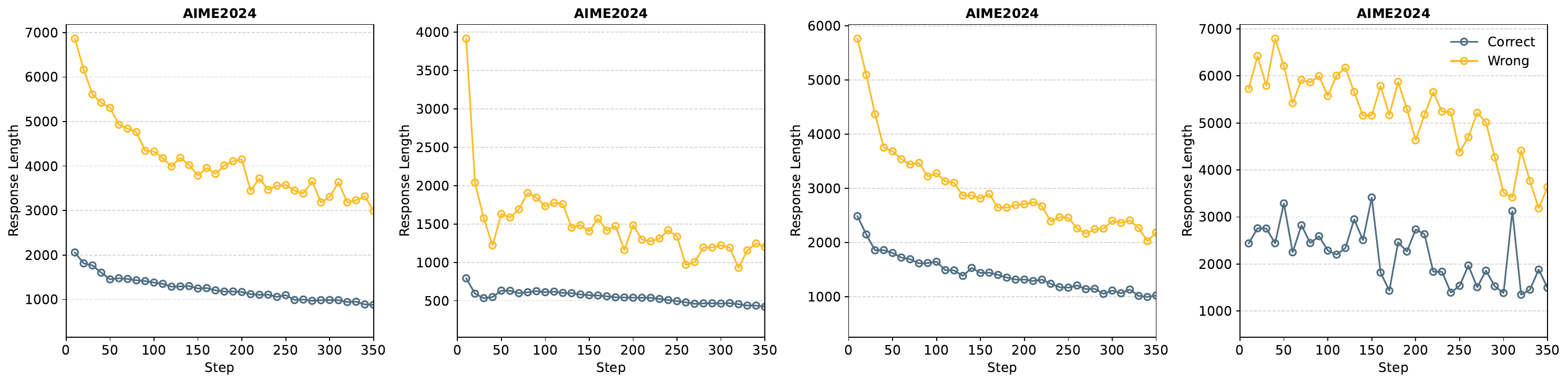}
    \caption{\textbf{Response length dynamics for correct vs. wrong samples during training.} The x‑axis indicates training steps, and the y‑axis denotes response length in tokens for models trained on DeepSeek-R1-Distill-Qwen-1.5B as the base model. The blue line tracks correctly answered samples, while the yellow line tracks incorrectly answered samples. In the early stages, incorrect samples produce substantially longer responses, reflecting the effect of our length‑incentive mechanism. As the dynamic length reward gradually diminishes, the response length for incorrect samples falls more sharply than that for correct samples, illustrating the model’s adaptive pruning of redundant reasoning steps.  
}
    \label{fig:correct_wrong}
\end{figure*}

Figure~\ref{fig:correct_wrong} illustrates how response length evolves for correct and incorrect samples under our approach. In the early phase of training, incorrect samples produce markedly longer outputs than correct ones, demonstrating the impact of our length‑incentive mechanism in promoting thorough exploration on challenging cases. As the dynamic length reward takes effect around mid‑training, the length for wrong samples declines steeply—outpacing the reduction seen for correct samples—and the gap between the two curves narrows. By later stages, both curves converge toward similarly concise rationales, indicating that the model has learned to apply efficient reasoning uniformly. This behavior confirms that our combination of significance‑aware and dynamic rewards not only drives exploration where needed but also enforces brevity once sufficient understanding is achieved, resulting in a balanced, adaptive pruning of redundant tokens.  

To examine how response length relates to answer correctness, we compare output length distributions of our method and the Vanilla PPO baseline across four benchmarks using \textit{DeepSeek-R1-Distill-Qwen-1.5B}. As shown in Figure~\ref{fig:len_hist}, correct responses consistently exhibit shorter lengths than incorrect ones across all tasks. Our method further produces sharply concentrated distributions for correct samples, suggesting more focused and efficient reasoning. In contrast, Vanilla PPO outputs are generally longer and more dispersed, with substantial overlap between correct and incorrect cases. Notably, the length of incorrect samples is substantially reduced compared to Vanilla PPO, suggesting that the dynamic reward mechanism—which gradually penalizes verbosity during training—plays a role in guiding more efficient responses.

\begin{figure*}[t]
    \centering
    \includegraphics[width=0.95\textwidth]{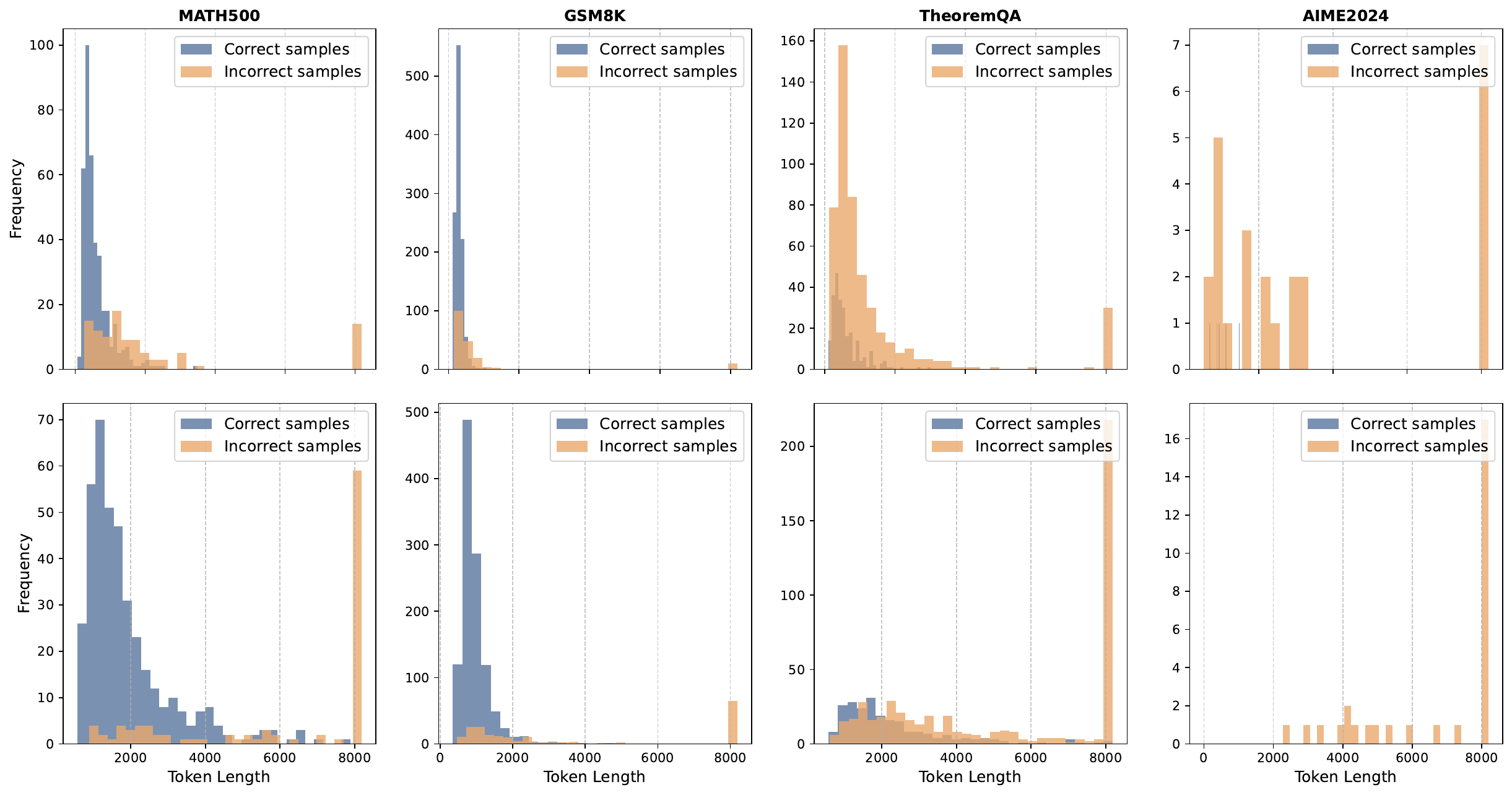}
    \caption{\textbf{Distribution of response lengths for correct vs.\ incorrect samples.} 
Histograms show the frequency of token lengths in model outputs across four benchmarks using DeepSeek-R1-Distill-Qwen-1.5B as the base model. 
Each plot compares correct responses (blue) and incorrect responses (orange). 
The top row corresponds to our method, while the bottom row shows results from Vanilla PPO. 
Across all datasets, incorrect samples are more likely to produce longer outputs, while correct samples tend to cluster in shorter length ranges. 
Compared to Vanilla PPO, our method produces a sharper, more compact distribution concentrated in shorter length regions.
}

    \label{fig:len_hist}
\end{figure*}

\section{Analysis of Incorrect Response Length under Different Reward Designs}
\label{ap:response_reward}

\begin{figure*}[t]
    \centering
    \includegraphics[width=0.95\textwidth]{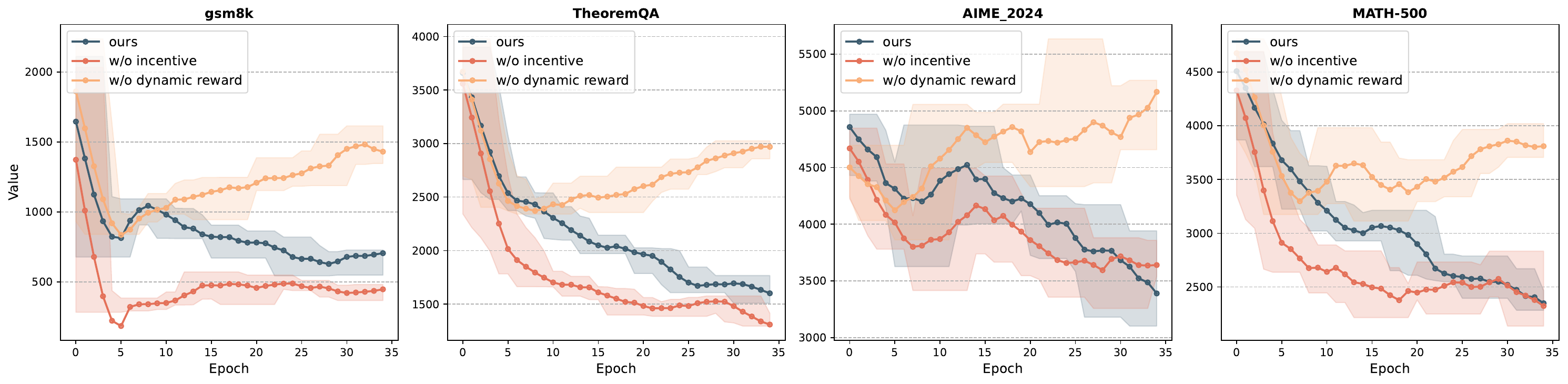}
    \caption{
\textbf{Effect of Reward Design on Incorrect Response Length.}
We visualize the average significant response length of incorrect predictions during training on four benchmarks using DeepSeek-R1-Distill-Qwen-1.5B as the base model. Compared to the variant without incentive, our full method produces longer responses for incorrect samples, suggesting that the significance-aware reward encourages more thorough exploration when the model is uncertain. In contrast, removing the dynamic reward leads to persistently longer outputs, whereas our full method shows a clear reduction in response length over time, confirming the effectiveness of dynamic reward scheduling in promoting concise reasoning. Together, these trends highlight the complementary roles of the two reward components in balancing exploration and efficiency.
}
    \label{fig:ap_ablation}
\end{figure*}

To gain deeper insights into the effect of our reward design during training, we analyze the significant response length of incorrect predictions—a proxy for the model’s exploratory behavior under uncertainty. As shown in Figure \ref{fig:ap_ablation}, our full method generates longer responses for incorrect samples compared to the version \textbf{without incentive} across all four benchmarks, indicating that the significance-aware reward successfully promotes deeper reasoning attempts when the model is uncertain. In contrast, removing the dynamic reward leads to consistently longer outputs throughout training, while our method gradually reduces response length over time. This trend confirms that dynamic reward scheduling effectively encourages concise reasoning as the model matures. Together, these findings highlight the complementary roles of the two reward components: one enhances exploration during failure, and the other improves efficiency during learning progression.

\section{Case Study}
\label{ap:case_study}

To further assess the reasoning quality and output efficiency of our method, we conduct two case studies comparing the Base model, PPO, and our \textsc{Bingo} framework across two model scales. As shown in Figures~\ref{fig:case-study-1.5b} and~\ref{fig:case-study-7b}, the Base model often produces excessively long and repetitive outputs, sometimes exceeding 5,000 words, and either fails to converge or hides the correct answer deep within verbose reasoning or terminated generations without a final answer. PPO improves conciseness but occasionally generates incorrect outputs or retains unnecessary repetition in intermediate steps. In contrast, \textsc{Bingo} consistently produces correct, well-structured solutions while significantly reducing output length—achieving up to 3x compression over PPO in the 7B setting. These examples demonstrate the effectiveness of our reward design in balancing reasoning depth and brevity, encouraging the model to generate focused and efficient reasoning even on complex tasks.

\begin{figure*}[ht]
    \centering
    \resizebox{0.8\textwidth}{!}{
    \fbox{\parbox[c]{1\textwidth}{
    
        \textbf{Problem: Find the greatest integer less than $(\sqrt{7} + \sqrt{5})^6.$  (Do not use a calculator!) } 

        \vskip 0.1in

        \xmark \, \textbf{Response of Base Model with 2509 words: } {\setlength{\fboxsep}{-1pt}
        
         \input{figures/case_study/1.5bbase}
        }

        \vskip 0.1in

        \xmark \, \textbf{Response of PPO with 403 words: } {\setlength{\fboxsep}{-1pt}
        
         \input{figures/case_study/1.5bppo}
        }

        \vskip 0.1in

        \cmark \,\textbf{Response of Bingo with 470 words: } {\setlength{\fboxsep}{-1pt}
        
         \input{figures/case_study/1.5bours}
        }

        \vskip 0.1in

        \textbf{Final Answer: } 13535.
    }}}%
    \caption{Case study under the DeepSeek-R1-Distill-Qwen-1.5B model with three settings: Base, PPO, and Bingo. Blue highlights some redundant and repetitive tokens, while red marks omitted content and the final answer.}

    \label{fig:case-study-1.5b}
\end{figure*}

\begin{figure*}[ht]
    \centering
    \resizebox{0.8\textwidth}{!}{
    \fbox{\parbox[c]{1\textwidth}{
    
        \textbf{Problem: In how many ways can 8 people sit around a round table if 3 of the people -- Pierre, Rosa, and Thomas -- all want to sit together?  (Two seatings are considered the same if one is a rotation of the other.) } 

        \vskip 0.1in

        \xmark \, \textbf{Response of Base Model with 5783 words: } {\setlength{\fboxsep}{-1pt}
        
         \input{figures/case_study/7bbase}
        }

        \vskip 0.1in

        \cmark \, \textbf{Response of PPO with 437 words: } {\setlength{\fboxsep}{-1pt}
        
         \input{figures/case_study/7bppo}
        }

        \vskip 0.1in

        \cmark \,\textbf{Response of Bingo with 155 words: } {\setlength{\fboxsep}{-1pt}
        
         \input{figures/case_study/7bours}
        }

        \vskip 0.1in

        \textbf{Final Answer: } 1.
    }}}
    \caption{
    Case study under the DeepSeek-R1-Distill-Qwen-7B model with three settings: Base, PPO, and Bingo. Blue highlights some redundant and repetitive tokens, while red marks omitted content and the final answer.
    }
    \label{fig:case-study-7b}
\end{figure*}

\section{Hyperparameter Study}
\label{ap:hyperparameter}



We evaluated several combinations of hyperparameters for \textsc{Bingo} on the GSM8K dataset using DeepSeek-R1-Distill-Qwen-1.5B as the base model. Table~\ref{tab:bingo-hparams} reports the accuracy and output length across different settings.

\begin{table}[h]
\centering
\small
\caption{Performance of \textsc{Bingo} under different hyperparameter settings on GSM8K.}
\label{tab:bingo-hparams}
\begin{tabular}{ccccccccc}
\toprule
$\lambda_c$ & $\lambda_w^{is}$ & $\lambda_w^{s}$ & $S$ & $\beta$ & $\alpha$ & $\tau$ & Acc. & Len. \\
\midrule
2 & 2 & 5 & 5 & 2   & 0.5 & 0.5 & 86.6 & 570 \\
2 & 2 & 5 & 10 & 5   & 0.2 & 0.8 & 86.7 & 585 \\
5 & 5 & 5 & 10  & 2.5 & 0.4 & 0.6 & 86.9 & 578 \\
\textbf{5} & \textbf{5} & \textbf{5} & \textbf{10} & \textbf{2.5} & \textbf{0.5} & \textbf{0.5} & \textbf{87.0} & \textbf{563} \\
\bottomrule
\end{tabular}
\end{table}

\noindent\textbf{Hyperparameter Definitions:}
\begin{itemize}[leftmargin=20pt, itemsep=0pt, labelsep=5pt, topsep=0pt]
    \item $\lambda_c$: Insignificant Length Reward Weight for Correct Samples.  
    \item $\lambda_w^{is}$: Insignificant Length Reward Weight for Incorrect Samples.  
    \item $\lambda_w^{s}$: Significant Length Reward Weight for Incorrect Samples.  
    \item $S$: Slope interval for the Dynamic Length Reward.  
    \item $\beta$: Threshold for Training Phase Transition.  
    \item $\alpha$: Decay Factor for Dynamic Length Reward.  
    \item $\tau$: Threshold for Significant Tokens.  
\end{itemize}

As shown in Table~\ref{tab:bingo-hparams}, the performance of \textsc{Bingo} remains stable, with both accuracy and output length exhibiting only minor fluctuations across the tested hyperparameter ranges. This indicates that the method is robust to hyperparameter choices. Since the last configuration achieves the best overall performance, we fixed these hyperparameters for methods and datasets to ensure consistency and fairness in comparison.


\section{Notation Table}
\label{ap:notation}

Table~\ref{tab:notation} offers a detailed overview of the notations utilized in this paper, along with their respective explanations. It serves as a handy reference to assist readers in grasping the concepts discussed in our work.

\begin{table*}[h]
\centering
\caption{Notation used throughout the paper}
\label{tab:notation}
\resizebox{0.8\textwidth}{!}{%
\begin{tabular}{cl}
\toprule
\textbf{Notation} & \textbf{Description} \\
\midrule
\multicolumn{2}{l}{\textit{General}} \\
$y$ & Sequence of tokens generated by the language model \\
$x$ & Input prompt for the language model \\
$n$ & Total length of the sequence $y$ \\
$y_i$ & $i$-th token in the generated sequence $y$ \\
$\hat{z}(y)$ & Extracted final answer from the generated sequence $y$ \\
$z$ & Ground-truth answer \\
$\mathbb{E}_{\pi_\theta}$ & Expectation over policy $\pi_\theta$ \\
$A(L)$ & Expected accuracy as a function of output length $L$ \\
$L$ & Length of the output sequence generated by the model \\
$L_{\max}$ & Maximum response length in the dataset \\
$\mathrm{Acc}$ & Exact match accuracy of the final output \\
$\text{L-Acc}$ & Length-normalized accuracy, defined as $\mathrm{Acc} \times \sqrt{1 - \frac{L}{L_{\max}}}$ \\
$S(y_i)$ & Significance score of token $y_i$ \\
$L^s$ & Number of significant tokens in the response \\
$L^{is}$ & Number of insignificant tokens in the response \\
$\tau$ & Threshold for classifying a token as significant or insignificant \\
\midrule
\multicolumn{2}{l}{\textit{Reinforcement Learning}} \\
$\pi_\theta$ & Policy parameterized by $\theta$ \\
$\hat{A}_t$ & Advantage estimate at time step $t$ \\
$r_t(\theta)$ & Importance sampling ratio for policy optimization \\
$R^\textsc{Bingo}$ & Reward function in the \textsc{Bingo} framework \\
$\mathcal{J}_{\textsc{Bingo}}(\theta)$ & PPO objective with \textsc{Bingo} reward function \\
$\epsilon$ & Clipping parameter in the PPO objective \\
\midrule
\multicolumn{2}{l}{\textit{Rewards and Penalties}} \\
$r_{is}(y)$ & Reward for insignificant tokens in sequence $y$ \\
$r_s(y)$ & Reward for significant tokens in sequence $y$ \\
$\lambda_c$ & Coefficient for penalty on correct responses \\
$\lambda_w$ & Coefficient for penalty on incorrect responses \\
$k$ & Dynamic scaling factor for length reward \\
$\alpha$ & Scaling factor for the decay in dynamic length reward \\
$\beta$ & Threshold for transition between exploration and compression in dynamic reward \\
\midrule
\multicolumn{2}{l}{\textit{Length Penalty}} \\
$L_{\text{ref}}^{is}$ & Reference number of insignificant tokens \\
$L_{\text{ref}}^s$ & Reference number of significant tokens \\
$k(t)$ & Dynamic scaling factor for adjusting length reward over time \\
\midrule
\multicolumn{2}{l}{\textit{Miscellaneous}} \\
$\mathcal{M}_e$ & Model used to estimate token significance (LLMLingua-2) \\
$\mathbbm{1}[\cdot]$ & Indicator function (1 if true, 0 otherwise) \\
$\mathcal{Y}_{\mathrm{sig}}$ & Set of significant tokens in the sequence $y$ \\
$\mathcal{Y}_{\mathrm{insig}}$ & Set of insignificant tokens in the sequence $y$ \\
\bottomrule
\end{tabular}%
}
\end{table*}

\end{document}